%% file: Journal-version.tex
\documentclass[journal]{IEEEtran}
\IEEEoverridecommandlockouts
\usepackage{cite}
\usepackage{amsmath,amssymb,amsfonts}
\usepackage{graphicx}
\usepackage{textcomp}
\usepackage{xcolor}
\def\BibTeX{{\rm B\kern-.05em{\sc i\kern-.025em b}\kern-.08em
    T\kern-.1667em\lower.7ex\hbox{E}\kern-.125emX}}
    
\usepackage{amsmath}
\usepackage{algorithm}
\usepackage{algorithmic}
\usepackage{hyperref}
\usepackage{tikz}
\usepackage{float}
\def\checkmark{\tikz\fill[scale=0.4](0,.35) -- (.25,0) -- (1,.7) -- (.25,.15) -- cycle;} 
\usepackage{graphicx}
\usepackage{xspace}

\usepackage{multirow}
\usepackage{caption}
\usepackage{subcaption}
\usepackage[compact]{titlesec}  
\titlespacing{\section}{2pt}{2pt}{2pt}
\DeclareMathOperator*{\argmin}{arg\,min}
\usepackage{amsthm}
\newtheorem{theorem}{Theorem}
\newtheorem{lemma}{Lemma}

\newtheorem{remark}{Remark}

\newtheorem{assumption}{Assumption}
\usepackage{ upgreek }
\usepackage{soul}
\begin{document}

\title{

Electrical Load Forecasting over Multihop Smart Metering Networks with Federated Learning 
}

 

\author{\IEEEauthorblockN{Ratun Rahman,  Pablo Moriano,~\IEEEmembership{Senior Member,~IEEE,}  Samee U. Khan,~\IEEEmembership{Senior Member,~IEEE,} and Dinh C. Nguyen, ~\IEEEmembership{Member,~IEEE,}}
\thanks{* Part of this work has been accepted at the IEEE Consumer Communications \& Networking Conference (CCNC), Jan. 2025 \cite{rahman2024electrical}.}
\thanks{Ratun Rahman and Dinh C. Nguyen are with the Department of Electrical and Computer Engineering, University of Alabama in Huntsville, Huntsville, AL 35899 (emails: \{rr0110,  dinh.nguyen\}@uah.edu). }
\thanks{Pablo Moriano is with the Computer Science and Mathematics Division, Oak Ridge National Laboratory, Oak Ridge, TN 37930 (e-mail: moriano@ornl.gov). }
\thanks{Samee U. Khan is with the Department of Electrical and Computer Engineering, Kansas State University, Manhattan, KS 66506 (email: sameekhan@ksu.edu). }

}
\maketitle
\pagenumbering{gobble} 



\begin{abstract}
Electric load forecasting is essential for power management and stability in smart grids. This is mainly achieved via advanced metering infrastructure, where smart meters (SMs) record household energy data. Traditional machine learning (ML) methods are often employed for load forecasting, but require data sharing, which raises data privacy concerns. Federated learning (FL) can address this issue by running distributed ML models at local SMs without data exchange. However, current FL-based approaches struggle to achieve efficient load forecasting due to imbalanced data distribution across heterogeneous SMs. This paper presents a novel personalized federated learning (PFL) method for high-quality load forecasting in metering networks. A meta-learning-based strategy is developed to address data heterogeneity at local SMs in the collaborative training of local load forecasting models. Moreover, to minimize the load forecasting delays in our PFL model, we study a new latency optimization problem based on optimal resource allocation at SMs. A theoretical convergence analysis is also conducted to provide insights into FL design for federated load forecasting. Extensive simulations from real-world datasets show that our method outperforms existing approaches regarding better load forecasting and reduced operational latency costs.


\end{abstract}

\maketitle

\begin{IEEEkeywords}
\textcolor{black}{Federated learning, load forecasting, multihop, smart grid, smart meter}
\end{IEEEkeywords}

\section{Introduction} \label{Sec:Introduction}

Electrical load forecasting is crucial for power management in smart grids. This service is mainly supported via advanced metering infrastructure, where smart meters (SMs) record household energy consumption and share this data with the server of the utility company \cite{9770488}. This enables utility providers to estimate future electricity demands and bolster grid reliability. Conventional load-forecasting techniques in machine learning (ML) and deep learning (DL) techniques utilize pattern-finding abilities to predict future outcomes. For example, long short-term memory (LSTM) has shown its potential for time-series data-based load forecasting applications \cite{hong2020deep, bouktif2018optimal}. \textcolor{black}{While smart grid technologies enhance efficiency and allow for real-time monitoring, they pose substantial security and privacy issues.  The constant data collection by smart meters (SMs) reveals precise residential energy usage habits that might be used for user profiling or monitoring.  This danger is amplified by centralized machine learning systems, which transport raw data to utility servers, allowing for eavesdropping, data breaches, and malicious inference attacks.  Furthermore, illegal access or manipulation of this data can result in inaccurate projections, financial fraud, and even operational disturbances in grid stability \cite{su2021secure, badr2023privacy}.  Thus, establishing safe and privacy-preserving learning frameworks is critical for practical implementation in smart grid systems.} In 2009, the compulsory roll-out of SMs in the Netherlands was halted following a court ruling that the metering data collection violated customers’ privacy rights \cite{cuijpers2013smart}. 


Recently, federated learning (FL) has been studied to address this data-sharing problem in load forecasting \cite{fekri2022distributed, taik2020electrical, gholizadeh2022federated}.  The local load forecasting model is trained at SMs using local metering data before sending it to the global server for the next global round. \textcolor{black}{This approach ensures user privacy by keeping raw data localized and preventing the transmission of sensitive information across networks.}  However, these literature works have struggled with addressing data heterogeneity, where they assume that every SM has a dataset with a similar data distribution. However, this is not realistic in real-world metering networks, where each SM typically owns a unique metering data distribution due to the nature of \textit{personalized} energy consumption patterns of households.  

To address this problem, we provide a novel load forecasting method for data heterogeneity in real-world metering networks. Our key idea is a new \textit{personalized federated learning (PFL)}-based load forecasting method. PFL handles data over-fitting by creating a customized load forecasting model for every SM. Our PFL technique is based on meta-learning, which helps local models to be trained properly by choosing the best parameters using the trial and run method \cite{9428530}. In this regard, each SMs participates in learning a custom load forecasting model, and they share local model parameters with the utility's server for model aggregation, aiming to build a global load forecasting model for the entire network with good generalization \cite{rahman2025multimodal, rahman2024improved}. 



Moreover, introducing FL into distributed load forecasting incurs latency costs due to model training at SMs and model communication between SMs and the utility's server. Minimizing the round-trip latency in such an FL-based load forecasting system is crucial to ensure timely load forecasting service of the entire metering network.  By addressing the latency issue, load forecasting efficiency and responsiveness can be enhanced for reliable smart grids. This will enhance the system's overall performance, allowing for more accurate and timely predictions, which are essential for effective energy management and distribution. \textit{This motivates us to jointly consider learning and latency optimization design for load forecasting to achieve optimal performance in terms of better accuracy and minimal delays of load forecasting.}




\subsection{Our Key Contributions} Motivated by the above limitations, \textit{we propose a novel load forecasting approach over metering networks in the anonymous grid.} 
Our key contributions are summarized as follows:
\begin{itemize}
\item We propose a new PFL approach called personalized meta-LSTM algorithm with a flexible SM participation method for collaborative load forecasting in the smart grid. This allows complicated and diverse data to be structured, assembled, and processed quickly, removing the need of data sharing to protect the privacy and security of household electricity recordings.  
\item We develop a personalized learning approach for local load forecasting in SMs based on meta-learning. Before training the local model, the clients are temporarily evaluated using varying learning rates. The most suitable learning rate is then selected among the available learning rates based on which one yields the lowest loss value. Next, we train local models with the optimal learning rate.
\item We propose a new latency optimization method to minimize the load forecasting delays caused by introducing PFL into the metering networks. The key objective is to find optimal resource allocation strategies for SMs, including transmit power and computational frequency, to optimize the round-trip PFL delay, achieved by an efficient convex optimization solution. 
\item We conduct extensive simulations on real-world datasets under independent and identically distributed (IID) and non-IID data settings, indicating that our approach outperforms existing works regarding better load forecasting and reduced operational latency costs. A theoretical convergence analysis is also conducted to provide insights into FL design for federated load forecasting. 
\end{itemize}

\subsection{Paper Organization}
The rest of the paper is structured as follows. In Section \ref{section: related work} we describe some related works and literature review and in Section \ref{Sec:SystemModel}, we present our system model, detailing the architecture and components of our proposed system. Section 
 \ref{Sec:PFLAlgorithmDesignforLoadForecasting} discusses the PFL algorithm design for load forecasting and convergence analysis in Section \ref{sec: convergence analysis}. Section 
 \ref{Sec:LatencyAnalysisForPFL-basedLoad Forecasting} presents the latency analysis of the PFL-based load forecasting system. We evaluate our simulation results and performance evaluations in Section 
 \ref{Sec:SimulationsandPerformanceEvaluation}. We present an in-depth outcomes analysis, comparing our proposed solutions to existing methods. Finally, Section  \ref{Sec:Conslusion} concludes the paper. 

\section{\textcolor{black}{Related Works}} \label{section: related work}
\textcolor{black}{
The two main categories of electric load forecasting techniques used in smart grids are statistical and artificial intelligence (AI)-based techniques. ARIMA \cite{tarmanini2023short}, exponential smoothing \cite{smyl2023drnn}, and regression-based forecasting \cite{madhukumar2022regression} are statistical approaches that depend on previous behavior and function well in stable situations.  However, they frequently fail to capture nonlinear load behavior and perform inadequately under changing conditions.}

\textcolor{black}{In contrast, AI-based approaches, such as LSTM, ANN, and SVM, can calculate and predict complicated and nonlinear electricity demand correlations.} The works in \cite{rafi2021short, bouktif2018optimal} used LSTM models for load estimation; however, such techniques need centralized data sharing, which may expose sensitive user information.  Furthermore, they impose significant computational and energy costs at the server \cite{jawad2018robust, ali2017ancillary}.  Recent research \cite{taik2020electrical, briggs2022federated} used federated learning (FL) to increase accuracy, maintain privacy, and reduce communication overhead.  FedAVG was introduced in another study in \cite{fekri2022distributed}, however, it continues to struggle with training on non-IID data due to heterogeneity in smart meter distributions. \textcolor{black}{FL has developed as a privacy-preserving smart metering system that allows local model training without exchanging raw consumption data.  This decentralized method reduces potential threats like data leakage, model inversion, and user profiling, which are typical in centralized load forecasting systems \cite{badr2023privacy, su2021secure}.  FL protects home privacy by keeping user data on-device while allowing for precise load forecasting in energy systems.}

Recently, PFL techniques have been considered to tackle the data heterogeneity issue in load forecasting. The study in \cite{9770488} proposed a PFL technique for load forecasting where each SM customizes a federated prediction model. Another work in \cite{10233242} introduced a Generative Adversarial Network (GAN) based differential privacy (DP) algorithm that included multi-task PFL. However, this solution increases computational complexity at the server in the load forecasting process. 


\begin{table}
\footnotesize
\centering
\caption{Comparison of our approach with existing load forecasting methods.}
\begin{tabular}{|p{2cm}||p{0.6cm}|p{0.6cm}|p{0.6cm}|p{0.6cm}|p{0.6cm}|p{0.6cm}|}
 \hline
 Objectives & \cite{rafi2021short, bouktif2018optimal} &  \cite{taik2020electrical, fekri2022distributed, briggs2022federated} & \cite{9770488} & \cite{10233242} & \cite{rahman2024electrical} & Our Approach\\
 \hline
 Handle uncertain and non-iid data & & & \checkmark & \checkmark & \checkmark&\checkmark\\
 \hline
 Include diverse SMs & \checkmark & & \checkmark & \checkmark & \checkmark & \checkmark\\
 \hline
 Adaptability to user change & \checkmark & \checkmark &  & \checkmark & \checkmark& \checkmark\\
 \hline
 Handle large dataset & & \checkmark & \checkmark & \checkmark & \checkmark& \checkmark\\
 \hline
 Maintain server complexity & & \checkmark & \checkmark & & \checkmark & \checkmark\\
 \hline
 Keeping data secured & & \checkmark & \checkmark & \checkmark & \checkmark & \checkmark\\
 \hline
 Latency minimization & & & & & & \checkmark\\
 \hline
 Practical multi-hop settings & & & & & & \checkmark\\
 \hline
 Convergence Analysis & & & & & &\checkmark\\
 \hline
\end{tabular}
\label{table:related_work_table}
\vspace{-5mm}
\end{table}

Moreover, several studies have concentrated on communication in smart grids and smart metering networks. The authors in \cite{kabalci2016survey} and \cite{barai2015smart} explained a smart grid, introduced its components, and presented the communication methods used, highlighting their advantages and shortcomings. It also surveyed smart grid integration, classified communication technologies, and outlined hardware and software security requirements. The work in \cite{kabalci2022design} proposed a smart metering infrastructure with DC and AC analog front ends and communication interfaces, and remote monitoring software for accurate and efficient measurement and transmission in microgrid and smart home applications. 
This work introduced a reconfigurable authenticated key exchange scheme using reconfigurable physical unclonable functions (PUFs) for secure and efficient smart grid communication. It offers advantages in computation and communication costs over current protocols.

Many studies have focused on improving the latency of  FL and addressed FL in multi-hop networks. For instance, \cite{mohasen2022federated} optimized model aggregation, routing, and spectrum allocation, while \cite{chen2022federated} introduced FedAir to mitigate communication impacts on FL performance. \cite{pinyoanuntapong2020fedair} used hierarchical FL with adaptive grouping, \cite{nguyen2022toward} aimed to reduce congestion by predicting future network topologies, and \cite{cash2023wip} examined jamming attacks on decentralized FL. Despite these efforts, latency minimization for  FL in multi-hop networks remains unaddressed. Single-hop networks often fail over large areas due to limited transmit power, whereas multi-hop networks provide better communication, coverage, and flexibility. Research on FL in multi-hop networks has focused on mesh networks. Still, it is crucial to consider scenarios with no direct links between non-consecutive nodes for worst-case analysis. \textcolor{black}{Our method is based on a joint design of a new PFL algorithm for collaborative load forecasting and a latency optimization solution for minimizing load forecasting delays in a multi-hop network setting}. We compare our approach with related works in Table~\ref{table:related_work_table}.

\section{System Model}
\label{Sec:SystemModel}

\subsection{Categories of Load Forecasting}
\textcolor{black}{
Electrical load forecasting is often divided into four groups based on the forecasting horizon: very short-term (VSTLF), short-term (STLF), medium-term (MTLF), and long-term (LTLF).  Each has a specific function and use different data sources and models, and they are explained in Table~\ref{tab: type_lf}.
}

\subsection{Overall System Architecture}
\begin{figure}[!t]
\centering
\includegraphics[width=3.2in]{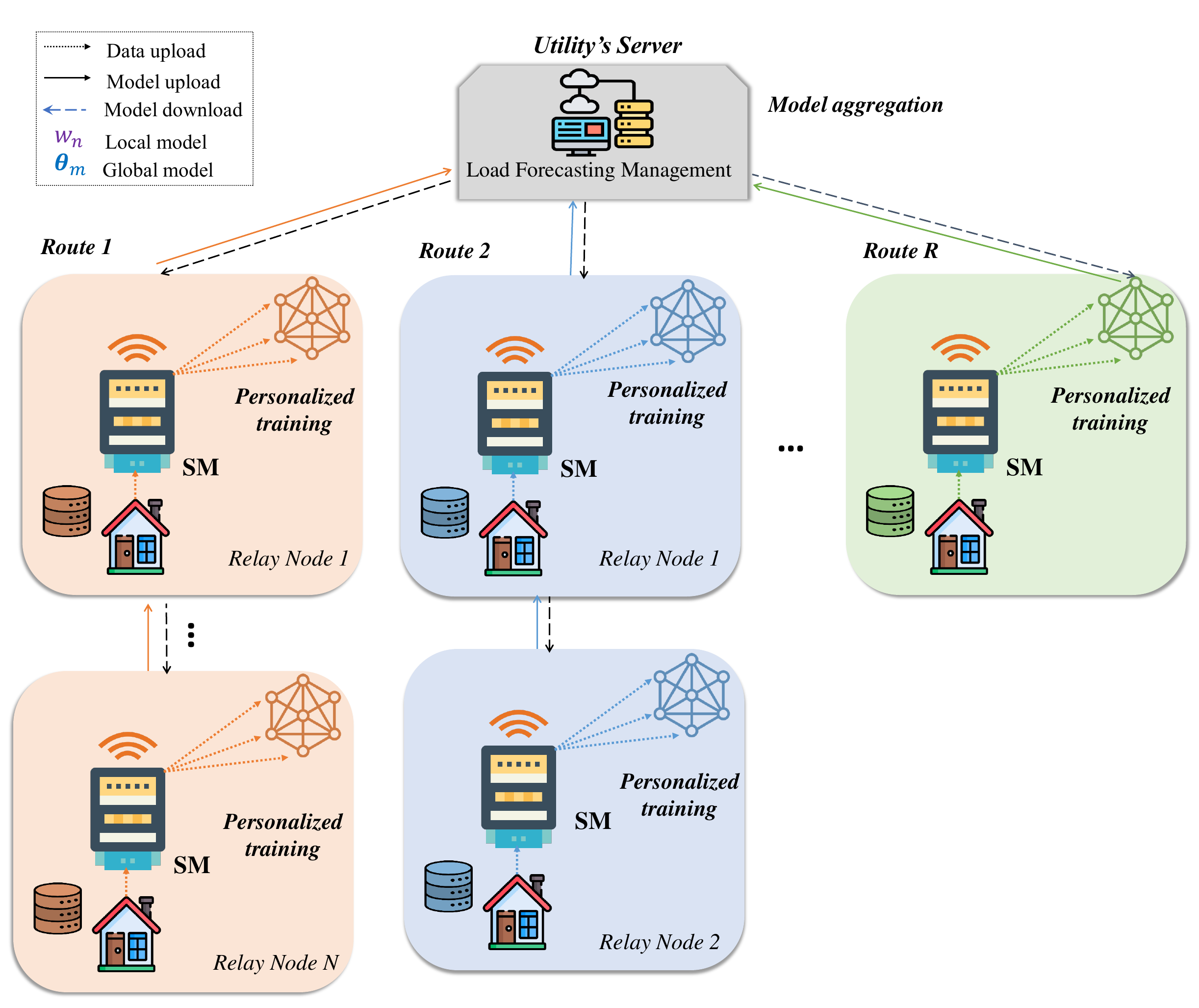}
\caption{Our proposed architecture for federated load forecasting in the multihop metering network. The SMs network is divided into different routes, each with a subset of SMs in a multi-hop topology. Each SM will train a custom load forecasting model and share the trained model with the utility's server for aggregation.}
\vspace{-1mm}
\label{Fig: Overview}
\vspace{-2mm}
\end{figure}

Fig.~\ref{Fig: Overview} illustrates the overall system for load forecasting over SMs. Inspired by the system model in \cite{miao2020evolutionary}, we consider a multi-hop metering network, where SMs are connected under a multi-hop topology in the wireless cellular network. Specifically, there are different routes and multiple SMs as relays on each route. Each SM trains a local load forecasting model and shares it with a utility server. The server is considered a global server where global model aggregation is performed based on a shared local model for load prediction. Each SM  denoted as $n \in \mathcal{N}$,  records household energy consumption data. Energy recordings are time-varying and differ over SMs. \textcolor{black}{We denote every global round as $k \in \mathcal{K}$ where $k = \{1, 2, 3, . . . , K\}$ and $K$ is the final global round. In each round $k$, each SM $n$ holds a local dataset $D_n^{(k)}$ that varies in every round and client, with size $|D_n^{(k)}|$. The SMs employ these datasets for global training in round $k$ and the total dataset is $D^{(k)} = \sum_{n\in \mathcal{N}} D_n^{(k)}$.}

\textcolor{black}{
To create a local model, each SM requires a gradient parameter denoted by $\nabla F$ with a learning rate $\alpha$ to fasten or slow the learning process. For our personalized approach, we have a series of learning rates ($\alpha_1,\alpha_2,\dots$), and we calculate the loss value for every learning rate for each SM $n$ in each round $k$ in a small dataset, as Dataloader denoted by $D^{temp}$. The learning rate that provides the lowest loss value is the optimal learning rate $(\alpha^{\text{best}}_{n,k})$. Then we use $\alpha^{\text{best}}_{n,k}$ for the local model training and send the local weight to the server for global model aggregation. Then the global model is updated for the next training round $k+1$ using federated averaging.}

\begin{table}
\color{black}
\centering
\footnotesize
\caption{Categories of Electric Load Forecasting}
\label{tab: type_lf}
\begin{tabular}{|p{0.75cm}|p{1.1cm}|p{5.8cm}|}
\hline
\textbf{Type} & \textbf{Horizon} & \textbf{Purpose} \\
\hline
VSTLF & Minutes to 1 hour & Ensures real-time grid stability and frequency regulation; supports automatic generation control. \\
\hline
STLF & 1 hour to 1 week & Aids in daily operations, unit commitment, and economic dispatch with weather and calendar data. \\
\hline
MTLF & 1 week to 1 year & Used for maintenance scheduling, fuel procurement, and mid-range market forecasting. \\
\hline
LTLF & Over 1 year & Supports infrastructure planning, investment decisions, and long-term energy policy development. \\
\hline
\end{tabular}
\end{table}

\subsection{Objective Function}
{\color{black}
Our proposed approach aims to achieve learning personalization for local load forecasting at SMs. In doing so, it is important to find an optimal learning rate for the local ML model, which is obtained by minimizing the  objective function:
\begin{equation} 
    F_{avg} = \frac{1}{K} \sum_{k=1}^{K} \mathcal{L}(y^k,\hat y^k),
\end{equation}
where $\mathcal{L}$ is a loss function. $y^k$ represents the actual value, and $\hat y^k$ is the predicted value for the $k^{th}$ task. In our work, we have used MAE and RMSE loss, the most widely used loss function for load forecasting.

\textbf{MAE Loss.} We train the model to minimize the mean absolute error (MAE), suitable for applications where all errors are equally important \cite{willmott2005advantages}, given by
\begin{equation}
    \mathcal{L}(y^k,\hat y^k) = \frac{1}{n} \sum_{i=1}^{N} |y_i - \hat{y}_i|.
\end{equation}

\textbf{RMSE Loss.} We calculate root mean squared error (RMSE), preferably when large deviations are particularly undesirable \cite{chai2014root}, using the equation
\begin{equation}
    \mathcal{L}(y^k,\hat y^k) = \sqrt{\frac{1}{N} \sum_{i=1}^N (y_i^k-\hat y_i^k)^2}.
\end{equation}}

\subsection{LSTM}
LSTM is usually used in time sequences and long-range dependencies datasets. To forecast future values based on past data, load prediction usually involves finding patterns and trends across time.  LSTM is a form of recurrent neural network (RNN) architecture consisting of unique units or memory cells designed to retain their state over time and regulate the processing and storage of information. \textcolor{black}{As a result, LSTM can handle long-term dependency problems better than RNN, which is crucial for load forecasting}. Each LSTM unit has three different gates that facilitate input, forget, and output gates, and two components: cell state and hidden state. 
At each training step $j$, the proposed LSTM model operates through the following key stages.
\textcolor{black}{
\begin{enumerate}
    \item The data is extracted from the cell state determined by the forget gate ($f_j$).
    \begin{equation}
        f_j = \sigma\!\left(W_f \begin{pmatrix} h_{j-1} \\ x_j \end{pmatrix} + b_f\right),
    \end{equation}
    where $\sigma$ is the sigmoid activation function, $W_f$ and $b_f$ are the weight matrix and bias for the forget gate $f_j$, $h_{j-1}$ is the hidden state from the previous time step, and $x_j$ is the input at the current time step.
    \item The input gate ($i_j$) determines additional data that must be added to the cell state $c_j$.
    \begin{equation}
        i_j = \sigma\!\left(W_i \begin{pmatrix} h_{j-1} \\ x_j \end{pmatrix} + b_i\right).
    \end{equation}
    \item A fresh candidate value to be added to the cell state is provided by the candidate cell state $\tilde C_j$.
    \begin{equation}
        \tilde{C}_j = \tanh\!\left(W_c \begin{pmatrix} h_{j-1} \\ x_j \end{pmatrix} + b_c\right),
    \end{equation}
    where $tanh$ is the hyperbolic tangent activation function and $W_c$ and $b_c$ are the weight matrix and bias for the candidate cell state $\tilde C_j$. 
    \item Cell state update combines the old cell state, forget gate output, input gate output, and candidate cell state to change the cell state.
    \begin{equation}
        C_j = f_j \odot C_{j-1} + i_j \odot \tilde{C}_j,
    \end{equation}
    where $C_{j-1}$ is the cell state from the previous time step. 
    \item From the current cell state, the output gate $o_j$ decides what information to output.
    \begin{equation}
        o_j = \sigma\!\left(W_o \begin{pmatrix} h_{j-1} \\ x_j \end{pmatrix} + b_o\right).
    \end{equation}
    \item The hidden state $h_j$ is updated for the current time step.
    \begin{equation}
        h_j = o_j \odot \tanh(C_j).
    \end{equation}
\end{enumerate}
    }

\section{PFL Algorithm Design for Load Forecasting} \label{Sec:PFLAlgorithmDesignforLoadForecasting}
Fig. \ref{Fig: Overview} depicts the multihop load forecasting framework using PFL, where a centralized server is connected to relay and leaf nodes. Each node goes through a local training process before sending its local model to its parent node. We can separate our system model into multiple steps as follows.

\textbf{Step 1:} 
We assume there are $N$ SMs and an initial parameter of the global model $\boldsymbol{w}_0$. A generalized FL with a single server for global round $k \in \mathcal{K}$ can be explained as

\begin{equation}
\min_{\boldsymbol{w} \in {\rm I\!R}^d} F_k(\boldsymbol{w_k}) := \frac{1}{N} \sum_{n=1}^{N} f_{n,k}(\boldsymbol{w}_{n,k}),
\end{equation}
where the function $f_i: \rm I\!R^d \longrightarrow \rm I\!R, n\in \mathcal{N}=\{1,2,3,\dots,N\}$ denotes the predicted loss value over $m^{th}$ SM's data distribution:
\begin{equation}\label{supervised_ML}
f_{i,k}(\boldsymbol{w}_n) := {\rm I\!E}_{\xi_i} \left [ f_{n,k}^{'} (\boldsymbol{w}_n,x_n) \right ].
\end{equation}
Here, ${f_{i,k}}^{'} (\boldsymbol{w}_n,x_n)$ is a loss function calculating the difference between data sample $x_n$ and its corresponding using $\boldsymbol{w}_n$ at round $k$.

\textbf{Step 2:} 
Assume that local training iteration index is denoted as $j \in \mathcal{J}$,  where $j = \{1, 2, 3, . . . , J\}$, the the local update at SM $n$ using LSTM is expressed as: 
\begin{equation} \label{eq: sgd}
\boldsymbol{w}_{n,k}^{j+1} = \boldsymbol{w}_{n,k}^{j} - \alpha_k \nabla F(\boldsymbol{w}_{n,k}^j,\chi_{n,k}^j),
\end{equation}
where $\alpha >0$ is the local learning rate, and $\chi$ is the non-IID sample from the local dataset. However, the learning rate $\alpha_k$ is not constant in our approach. We use optimal and personalized learning rate (denoted as $\alpha^{\text{best}}_{n,k}$) instead for SM $n$ and round $k$. 

\textbf{Step 3:}
To calculate $\alpha^{\text{best}}_{n,k}$, in each round $k$ for every SM $n$, we apply a group of available learning rates given as $\alpha_j$ on Dataloader $(D_{n,k}^{temp})$ where $j$ is the total number of available learning rates to calculate the loss value. The loss value at $i \in j$ is calculated as:
\begin{equation}
    f_{i,k}^{'} (\alpha_i) = f_{i,k}((\boldsymbol{w}_{n,k},\alpha_i), D_{n,k}^{temp}).
\end{equation}
Then we select the $\alpha_i$ as $\alpha^{\text{best}}_{n,k}$ that produce the minimum $f_i^{'}$ value. So, we can explain that as:
\begin{equation}
    \alpha^{\text{best}}_{n,k} := \argmin_{i \in j} (f_{i,k}^{'}(\alpha_i)).
\end{equation}

\textbf{Step 4:}
We calculate the local model training in each SM $n$ (a leaf or relay SM) using equation \ref{eq: sgd}. So after receiving the parameter of the global model $\boldsymbol{w}_{n,k}$, every $n$ updates its personalized model by using the optimized learning rate $\alpha^{\text{best}}_{n,k}$. So, we can express equation \ref{eq: sgd} as:
\begin{equation}
\boldsymbol{w}_{n,k}^{j+1} = \boldsymbol{w}_{n,k}^{j} - \alpha^{\text{best}}_{n,k} \nabla F(\boldsymbol{w}_{n,k}^j,\chi_{n,k}^j).
\end{equation} 

\textbf{Step 5:}
After $J$ rounds, each $n$ then sends its local model's weight $(\boldsymbol{w}_{n,k} = \boldsymbol{w}_{n,k}^J)$ to its parent node. If the node is not a leaf node, the current node relays its local and child models to the parent node. The weights eventually arrive at the server when the nodes are aggregated.

\textbf{Step 6:}
Once the server collects all the nodes' weights, it calculates the federated averaging for the next global round $k+1$ as:
\begin{equation}
    \boldsymbol{w}_{k+1} = \frac{1}{N} \sum_{n\in\mathcal{N}} \boldsymbol{w}_{n,k}. 
\end{equation}

\textbf{Step 7:}
We calculate the loss value for the updated weight on testloader $(D^{test})$ and then broadcast the updated weight to all the SMs. We do this loop for $K$ times and get the optimal global load forecasting model $w^*$. 

\begin{algorithm}
\footnotesize
	\caption{{Proposed PFL algorithm for high-quality load forecasting across SMs}}
	\begin{algorithmic}[1]
		\label{algo: metaSGD}
		\STATE \textbf{Input:}  The set of global communication rounds $\mathcal{K}$, local training round $\mathcal{J}$, a set of SMs $\mathcal{N}$
		\STATE \textbf{Initialization:} Initialize global model $\boldsymbol{w}_0$, different learning rates $\alpha_{0,1,\dots,j}$
		\FOR{each global communication round $k \in \mathcal{K}$}
		\STATE Send $\boldsymbol{w}_k$ to sampled SMs
		\FOR{each sampled SM $n \in \mathcal{N}$ in parallel}
		\FOR{each local training epoch $j \in \mathcal{J}$}
        \STATE Get $\boldsymbol{w}_k$ 
        \FOR{each learning rates $\alpha_i$ where $i \in j$}
        \STATE Calculate $f_{i,k}^{'} (\alpha_i) = f_{i,k}((\boldsymbol{w}_{n,k}^j,\alpha_i), D_{n,k}^{temp})$ on $D_{n,k}^{temp})$
        \STATE Save the best learning rate as $(\alpha^{\text{best}}_{n,k})$ that has the lowest $f_{i,k}^{'}$
        \STATE Return $\alpha^{\text{best}}_{n,k}$ to the local model $\theta_n$
        \ENDFOR
        \STATE Perform local model training (meta-learning) on $\theta_i$, $\boldsymbol{w}_{n,k}^{j+1} = \boldsymbol{w}_{n,k}^{j} - \alpha^{\text{best}}_{n,k} \nabla F(\boldsymbol{w}_{n,k}^j,D_{n,k}^{train})$
		\ENDFOR
        \STATE Send $\boldsymbol{w}_{n,k}$ to the server
		\ENDFOR
		\STATE The utility's server updates the global parameter by averaging: $\boldsymbol{w}_{k+1} = \frac{1}{N} \sum_{n\in\mathcal{N}}\boldsymbol{w}_{n,k}$ 
        \STATE Perform test on the updated weight $\boldsymbol{w}_{k+1}$
		\STATE The utility's server  broadcasts the aggregated global model $\boldsymbol{w}_{k+1}$ to all participating SMs for the next round of training
		\ENDFOR
        \STATE \textbf{Output:} Optimal global load forecasting model $\boldsymbol{w}^*$
  \end{algorithmic}
\end{algorithm}
The proposed PFL approach is summarized in Algorithm~\ref{algo: metaSGD}. For each global round $k$, after SM obtains the initial global weight from the utility server (line 7), they perform meta-learning functionalities in lines 8-12. The loss value is calculated for every available learning rate $\alpha_j$ (line 9). The $\alpha_j$ that produces the lowest loss value is then returned to the local model as the optimal learning rate in line 11. Then in line 13, we perform the local model training and send the updated local model to the server (line 15) after $T$ rounds of local rounds. Then the server does federated averaging on line 17, testing on line 18, and then saves the updated weight for the next global round. Finally, after $K$ global rounds, we get our optimal global model $\boldsymbol{w}^*$. 

\textcolor{black}{This PFL technique uses client-specific learning rates, allowing each SM to train its local model based on its unique data properties and learning capabilities.  SMs with high-quality or representative datasets—and those that perform better—can use higher learning rates to contribute more significantly to the global model.  On the other hand, SMs with less useful or noisy input employ lower learning rates, limiting their effect and preventing global performance degradation.  This adaptive and personalized updating technique guarantees that the global model receives more trustworthy information while accommodating all clients.  As a result, the aggregated model receives more accurate and balanced updates, which improves overall convergence and prediction accuracy while maintaining adaptation.}
\textcolor{black}{Furthermore, by speeding convergence on high-performing SMs and avoiding superfluous updates from weak nodes, our PFL technique minimizes the number of global rounds required to achieve 95\% final accuracy.  In our studies, normal FL took 80 rounds, but our technique achieved the same accuracy in 45 rounds, resulting in a 43.75\% decrease in global communication rounds.  This results in shorter training times and faster load forecasting, which is especially useful for latency-sensitive smart grid applications.}

\subsection{Limitations}
\textcolor{black}{Despite its potential, the proposed meta-learning-based PFL approach has some limitations.  It has significant computing costs due to the frequent local updates and many gradient steps necessary for client-specific adaptation.  This method not only increases energy consumption on resource-constrained SMs, but it also raises the danger of overfitting when local datasets are limited or non-representative, lowering generalization capabilities.  Furthermore, real-world implementation presents practical problems such as varying network capacity, high latency, and intermittent connectivity, which could prevent timely and effective model synchronization among clients.  SMs' hardware restrictions, such as limited memory and compute capability, make it much more difficult to train and update complicated models.  These characteristics influence the system's stability, responsiveness, and scalability in real-world scenarios.}

\textcolor{black}{While FL protects data privacy by storing them locally, sending models may still leak critical information via model upgrades \cite{badr2023privacy}. This is particularly concerning in non-IID environments when client-specific patterns might be identified.  Our present solution does not use formal privacy-preserving procedures like differential privacy or secure aggregation, which might improve protection while introducing additional complexity. FL is vulnerable to adversarial, backdoor, and model poisoning attacks.}
\section{Convergence Analysis} \label{sec: convergence analysis}
In our proposed framework, SMs exchange ML models, where the global model aggregation is executed. From this observation, we focus on analyzing the convergence properties of the federated model training. To support our convergence analysis, we introduce a virtual variable as $\bar{\boldsymbol{w}}_k^j = \frac{1}{N} \sum_{n\in\mathcal{N}} \boldsymbol{w}_{n,k}^j $, where $k \in \mathcal{K}$ denotes the global round. Accordingly, we also define $g_k^j = \frac{1}{N} \sum_{n\in\mathcal{N}} \nabla F_n(\boldsymbol{w}_{n,k}^j,\chi_{n,k}^j)$, and $h_k^j = \frac{1}{N} \sum_{n\in\mathcal{N}} \nabla F_n({x}_{n,k}^j,\zeta_{n,k}^j)$. It is easy to observe that $\bar{\boldsymbol{w}}_k^{j+1}  = \bar{\boldsymbol{w}}_k^j - \alpha_kg_k^j + \bar{\textbf{v}}_k^j$, and $\mathbb{E}g_k = \bar{g}_k$, $\mathbb{E}h_k = \bar{h}_k$, where $\mathbb{E}$ represents function's expectation. Before analyzing the convergence, we make the following common assumptions:
 \begin{assumption}
 Each local loss function $F_n$ ($n\in \mathcal{N}$) is $L$-smooth, i.e., $F_n(\boldsymbol{w}') - F_n(\boldsymbol{w}) \leq \langle \boldsymbol{w}'- \boldsymbol{w}, \nabla F_n(\boldsymbol{w}) \rangle + \frac{L}{2} ||\boldsymbol{w}'- \boldsymbol{w}||, \forall \boldsymbol{w}', \boldsymbol{w}$.
 \end{assumption}
 \begin{assumption}
 Each local loss function $F_n$ ($n\in \mathcal{N}$) is $\mu$-strongly convex, i.e., $F_n(\boldsymbol{w}') - F_n(\boldsymbol{w}) \ge \langle \boldsymbol{w}'- \boldsymbol{w}, \nabla F_n(\boldsymbol{w}) \rangle + \frac{\mu}{2} ||\boldsymbol{w}'- \boldsymbol{w}||, \forall \boldsymbol{w}', \boldsymbol{w}$.
 \end{assumption}
 \begin{assumption} \label{Assump:Variance-gradient}
 The variance of stochastic gradients on local model training at each SM is bounded: $\mathbb{E}||\nabla F_n(\boldsymbol{w}_{n,k}^j,\chi_{n,k}^j) - \nabla F_n(\boldsymbol{w}_{n,k}^j)||^2 \leq \sigma_r^2$. 
 \end{assumption}
We next introduce several lemmas employed in the main result of Theorem~\ref{theorem_covergence}.
\begin{lemma}
Let Assumption \ref{Assump:Variance-gradient} hold, the expected upper bound of the variance of the stochastic gradient on local model training is given as $\mathbb{E} ||g_k^j- \bar{g}_k^j||^2 \leq \frac{\sigma_r^2}{N^2}$. 
\end{lemma}
\begin{proof}
See Appendix \ref{Lemma_SGD:upperbound}. \renewcommand{\qedsymbol}{}
\end{proof}
\begin{lemma}
The expected upper bound of the divergence of $\boldsymbol{w}_{n,k}^j$ is given as 
\begin{equation} 
\begin{aligned}
& \left[ \frac{1}{N}\sum_{n\in\mathcal{N}}\mathbb{E} \Big\Vert\bar{\boldsymbol{w}}_k^j-\boldsymbol{w}_{n,k}^j\Big\Vert^2  \right] \leq  4\alpha_kJ B^2,
\end{aligned}
\end{equation}
for some positive $B$.
\end{lemma}
\begin{proof}
See Appendix \ref{Lemma_lemma2:upperbound}. \renewcommand{\qedsymbol}{}
\end{proof}
\begin{lemma}
The expected upper bound of $\mathbb{E} \left[||\bar{\boldsymbol{w}}_k^{j+1} - \boldsymbol{w}^*||^2 \right]$ is given as
\begin{equation} 
\footnotesize
\begin{aligned}
&\mathbb{E}||\bar{\boldsymbol{w}}_k^{j+1} - \boldsymbol{w}^*||^2  \leq 2(1-\mu\alpha_k)\mathbb{E}||\bar{\boldsymbol{w}}_k^j - \boldsymbol{w}^*||^2  
\\&+ \left(2+\frac{1}{2\alpha_k}\right) \left[ \frac{1}{N}\sum_{n\in\mathcal{N}} \mathbb{E} \Big\Vert\bar{\boldsymbol{w}}_k^j-\boldsymbol{w}_{n,k}^j\Big\Vert^2  \right]  +2\alpha_k^2\mathbb{E}||g_k^j - \bar{g}_k^j||^2.
\end{aligned}
\end{equation}
\end{lemma}
\begin{proof}
See  Appendix \ref{Lemma_lemma3:SGDupdate}. \renewcommand{\qedsymbol}{}
\end{proof}
\begin{theorem} \label{theorem_covergence}
Let Assumptions 1-3 hold, then the upper bound of the convergence rate of the federated model training at each cluster after $K$ global rounds satisfies
\begin{equation} \label{equa:final_convergenceIID0}
\footnotesize
\begin{aligned}
&\mathbb{E}\left[F_n(\boldsymbol{w}_K)\right] -F^* \leq \frac{L(1+L/\mu)}{\mu}\frac{1}{(K+L/\mu)} (F_n(\boldsymbol{w}_1) -F^*)
 \\&+\frac{16L}{30\mu^2(K+L/\mu)}\sum_{k=1}^{K} \left[4JB^2 \left(\frac{\alpha_k+1}{\alpha_k^2} \right) +\frac{\sigma_r^2}{N^2} \right],
\end{aligned}
\end{equation}
where $J$ is the number of local SGD rounds at each SM, $N$ is the number of SMs, $\alpha_k$ is the learning rate of each SM in global round $k$, and $L, B, \mu$ are constants. 
\end{theorem}

\begin{proof}
See  Appendix \ref{Proof_globalbound}. \renewcommand{\qedsymbol}{}
\end{proof}
\vspace{-2mm}
\begin{remark}
Theorem~\ref{theorem_covergence} implies an inverse relation between the overall FL convergence loss rate and global rounds $K$ and the number of SMs $N$ under a certain number of local SGD rounds $J$. That is, longer training rounds $K$ with more SMs $N$ involved in the training will decrease the upper bound's first and second terms, resulting in improved global model performance.
\end{remark}

\section{Latency Analysis For PFL-based Load Forecasting} \label{Sec:LatencyAnalysisForPFL-basedLoad Forecasting}
This section explicitly analyzes the latency in our PFL-based load forecasting system. 

\subsection{ Formulation of Latency Problem}

As shown in Fig.~\ref{Fig: Overview}, the network consists of leaf and relay nodes. We assume $R$ routes, each starting from a leaf node and passing through $N$ relay nodes to the utility server. Thus, the number of leaf nodes equals the number of routes. The set of leaf nodes is $\mathcal{R}={1, 2, \dots, R}$ and relay nodes per route are $\mathcal{M}={1, 2, \dots, M}$. All nodes forward data as follows. Leaf nodes train local models and upload to relay nodes; relay nodes train and relay received models. Each node uses its dataset, ensuring data privacy.

In the case of leaf node $r$, let $f_{r}$, $D_{r}$, and $C_{r}$ represent its CPU computation capability (in CPU cycles per second), the number of data samples, and the number of CPU cycles needed to process a data sample, respectively. If $L_{r}$ is the number of local iterations, the computation time $T_{r}^{\text{train}}$ for $L_{r}$ iterations is calculated as $T_{r}^{\text{train}}=\frac{L_{r}C_{r}D_{r}}{f_{r}}$. The corresponding energy consumption, $E_{r}^{\text{train}}$, is given by $E_{r}^{\text{train}}=L_{r}\zeta_{r}C_{r}D_{r}{f_{r}^{2}}$, where $\zeta_{r}$ is the effective switched capacitance that depends on the hardware and chip architecture of leaf node $r$. Upon completion of local computation, each SM uploads its local model to the parent. We consider frequency division multiple access for the up-link operation. The achievable rate $R_{r}$ of leaf node $r$ is calculated as $R_{r}=b_{r}\log_{2}\left(1+\frac{p_{r}g_{r}}{b_{r}n_{0}}\right)$, where $b_{r}$ represents the allocated bandwidth, $p_{r}$ is the transmit power, $g_{r}$ stands for the channel gain of leaf node $r$, and $n_{0}$ denotes the noise power spectral density. Assuming a constant data size $s$ for the local models, the uploading time can be expressed as $T_{r}^{\text{up}}=\frac{s_r}{R_{r}}$, and the corresponding energy consumption is $E_{r}^{\text{up}}=T_{r}^{\text{up}}p_{r}$ that differs on every $r$. Hence, the total time $T_{r}$ required for computing and uploading local models for leaf node $r$ is $T_{r}=T_{r}^{\text{train}}+T_{r}^{\text{up}}$. If the total energy consumed by leaf node $r$ for computing and uploading local models during each global iteration is denoted by $E_{r}$, it can be expressed as $E_{r}=E_{r}^{\text{train}}+E_{r}^{\text{up}}$.

In case of relay nodes, the computation time for $L_{m}$ local iterations is denoted by $T_{m}^{\text{train}}$, where $T_{m}^{\text{train}}=\frac{L_{m}C_{m}D_{m}}{f_{m}}$. Here, $f_{m}$ represents the CPU computation capability (in CPU cycles per second), $D_{m}$ stands for the number of data samples, and $C_{m}$ denotes the number of CPU cycles needed to process a data sample. The corresponding energy consumption by relay node $m$ is given by $E_{m}^{\text{train}}=L_{m}\zeta_{m}C_{m}D_{m}{f_{m}^{2}}$, where $\zeta_{m}$ is the effective switched capacitance of relay node $m$ that depends on the hardware and chip architecture of it. Like leaf nodes, relay nodes upload their local models to the server for aggregation after local computation. The uploading time for relay node $m$ is given by $T_{m}^{\text{up}}=\frac{s}{R_{m}}$, where $s$ represents the constant data size of the local models uploaded by all nodes. The achievable uploading rate $R_{m}$ is determined by $R_{m}=b_{m}\log_{2}\left(1+\frac{p_{m}g_{m}}{b_{m}n_{0}}\right)$, where $b_{m}$ stands for the allocated bandwidth, $p_{m}$ denotes the transmit power, and $g_{m}$ represents the channel gain of relay node $m$. The corresponding energy consumption is expressed as $E_{m}^{\text{up}}=T_{m}^{\text{up}}p_{m}$. In this work, all channels are assumed to have two fading effects that characterize mobile communications: large-scale and small-scale fading. The small-scale fading component is modeled using a Rayleigh distribution, while a deterministic path loss model represents the large-scale fading coefficient.

Furthermore, a relay node must transmit the local models of all preceding nodes. In a given route, relay node $m$ is connected to $m$ successor nodes: $(m-1)$ relay nodes and one leaf node. Let $T_{m}^{\text{tx}} = \sum_{k=1}^{m-1} T_{m,k}^{\text{tx}}$ be the time required to transmit the models of the $(m-1)$ preceding relay nodes. If $T_{m,r}^{\text{tx}} = \frac{s}{R_{m}}$ is the time to transmit the model of the preceding leaf node, the total transmission energy is $E_{m}^{\text{tx}} = E_{m}^{\text{up}} + (m-1)E_{m}^{\text{up}} = mE_{m}^{\text{up}}$, assuming equal model sizes. Energy for uploading a model of size $s$ depends on the node’s upload rate and transmit power. Thus, relay node $m$ consumes equal energy per preceding model. The total time for relay node $m$ per global iteration is $T_{m,r} = T_{m}^{\text{train}} + T_{m}^{\text{up}} + T_{m,r}^{\text{tx}} + T_{m}^{\text{tx}}$, and the total energy is $E_{m} = E_{m}^{\text{train}} + E_{m}^{\text{up}} + E_{m}^{\text{tx}}$.

If $T_{\text{total}}^{r}$ is the total time required for route $r$ (leaf node $r$ and relay node $1$ to $M$) to complete each global iteration, then it can be formulated as $T_{\text{total}}^r=(T_{r}+\sum_{m=1}^{M} T_{m,r})$. As the route that takes the longest time to complete each global iteration will be the bottleneck for the latency, the total time required for completing each global communication round denoted as $k$ ($k = [1,2,..., K]$) can be written as $T_{\text{total}} = \underset{r \in \mathcal{R}}{\max} T_{\text{total}}^r  = \underset{r \in \mathcal{R}}{\max} (T_{r}+\sum_{m=1}^{M} T_{m,r})$. Hence, the total latency of the FL system over $K$ global rounds can be expressed as
\begin{equation}
T_{\text{total}}^{\text{FL}} = \sum_{k=1}^{K}\left(\underset{r \in \mathcal{R}}{\max} T_{\text{total}}^r \right)  = \sum_{k=1}^{K}\left(\underset{r \in \mathcal{R}}{\max} (T_{r}+\sum_{m=1}^{M} T_{m,r})\right).
\end{equation}

This research aims to minimize the latency of the PFL-based load forecasting system. Based on the above analysis, we formulate the following optimization problem
\begin{subequations} 
\begin{align}
\min_{\pmb{p_{r}}, \pmb{f_{r}}, \pmb{p_{m}}, \pmb{f_{m}}} \quad & T_{\text{total}}^{\text{FL}} \label{eqn:2a}\\ 
\textrm{s.t.} \quad & 0 \le p_{r} \le P_{r}, \forall{r} \label{eqn:2b}\\
& 0 \le p_{m} \le P_{m}, \forall{m} \label{eqn:2c}\\
& 0 \le f_{r} \le F_{r}, \forall{r} \label{eqn:2d}\\
& 0 \le f_{m} \le F_{m}, \forall{m} \label{eqn:2e}\\
& E_{r} \le E_{r}^{\max}, \forall{r} \label{eqn:2f}\\
& E_{m} \le E_{m}^{\max}, \forall{m} \label{eqn:2g}
\end{align}
\end{subequations}
\noindent
where $\pmb{p_{r}}=\{p_{1}, p_{2}, \dots ,p_{R}\}$, $\pmb{p_{m}}=\{p_{1}, p_{2}, \dots ,p_{M}\}$, $\pmb{f_{r}}=\{f_{1}, f_{2}, \dots ,f_{R}\}$, and $\pmb{f_{m}}=\{f_{1}, f_{2}, \dots ,f_{M}\}$. In (18), \eqref{eqn:2b} and \eqref{eqn:2c} represent the feasible range of the transmit power due to the power budgets of the leaf nodes and the relay nodes. The CPU frequency of each node is constrained in \eqref{eqn:2d} and \eqref{eqn:2e}. The other two constraints, \eqref{eqn:2f} and \eqref{eqn:2g}, are on the energy consumption by each leaf node and relay node, respectively.

\subsection{Proposed Solution to Optimize Latency for the PFL-based Load Forecasting System}

Solving the problem in (18) is challenging due to the coupling of multiple optimization variables. The objective function \eqref{eqn:2a} as well as the energy constraints \eqref{eqn:2f} and \eqref{eqn:2g} are non-convex because of the $\log_{2}$ function of the achievable rates. As mentioned earlier, we divide the problem in (18) into two sub-problems to address the non-convex nature of the objective function and the constraints. Hence, the control variables of the problem in (18) are divided into two blocks: (i) the first block for leaf node optimization $(p_{r}, f_{r})$ and (ii) the second block for relay node optimization $(p_{m}, f_{m})$, which will be updated alternatively in an iterative fashion.

\textit{\textbf{For the first block}},  we introduce a new slack variable $x_{r}$ such that:

\begin{equation}
x_{r} \geq \frac{s}{b_{r}\log_{2}\left(1+\frac{p_{r}g_{r}}{b_{r}n_{0}}\right)}, \forall{r}. 
\end{equation} 
\textbf{Problem in (18)} can be equivalently re-written as
\begin{subequations} 
\begin{align}
\min_{p_{r},f_{r}} \quad & \sum_{k=1}^{K}\left[\max_{r \in \mathcal{R}} \left({\frac{L_{r}C_{r}D_{r}}{f_{r}}}+x_{r}\right.\right. \nonumber\\
& \left.\left.+\sum_{m=1}^{M}\left({\frac{L_{m}C_{m}D_{m}}{f_{m}}}+\frac{(m+1)s}{b_{m}\log_{2}\left(1+\frac{p_{m}g_{m}}{b_{m}n_{0}}\right)}\right)\right)\right] \label{eqn:4a}\\ 
\textrm{s.t.} \quad & L_{r}\zeta_{r}C_{r}D_{r}{f_{r}^{2}}+x_{r}p_{r} \le E_{r}^{\max}, \forall{r} \label{eqn:4b}\\
& \frac{s}{b_{r}x_{r}} \leq \log_{2}\left(1+\frac{p_{r} g_{r}}{b_{r} n_{0}}\right), \forall{r} \label{eqn:4c}\\
& \eqref{eqn:2b}, \eqref{eqn:2d}. \label{eqn:4d}
\end{align}
\end{subequations}

The objective \eqref{eqn:4a} is convex, while the constraint in \eqref{eqn:4d} is also convex. Hence, we now convexify the constraints \eqref{eqn:4b} and \eqref{eqn:4c}.

\textit{Constraint \eqref{eqn:4b}}: For $x_{r}>0$ and $p_{r}>0$, we apply \textcolor{black}{successive convex approximation (SCA)} to approximate $x_{r}p_{r}$ as

\begin{equation} \label{eqn:5}
x_{r} p_{r} \le \frac{1}{2} \frac{p_{r}^{i}}{x_{r}^{i}} x_{r}^{2} + \frac{1}{2} \frac{x_{r}^{i}}{p_{r}^{i}} p_{r}^{2} = h_{r}^{i}(x_{r},p_{r})
\end{equation}

\noindent
where $p_{r}^{i}$ and $x_{r}^{i}$ are the feasible point of $p_{r}$ and $x_{r}$ at iteration $i$. Hence, constraint \eqref{eqn:4b} can be convexified as

\begin{equation} \label{eqn:6}
    L_{r}\zeta_{r}C_{r}D_{r}{f_{r}^{2}} +\frac{1}{2} \frac{p_{r}^{i}}{x_{r}^{i}} x_{r}^{2} + \frac{1}{2} \frac{x_{r}^{i}}{p_{r}^{i}} p_{r}^{2} \le E_{r}^{\max}, \forall{r}.
\end{equation}

\textit{Constraint \eqref{eqn:4c}}: We use this inequality \cite{kabalci2022design}
\begin{align} \label{eqn:34}
\ln(1+z) \ge \ln(1+z_{i}) + \frac{z_{i}}{z_{i}+1} - \frac{(z_{i})^{2}}{z_{i}+1} \frac{1}{z}.
\end{align}
Now we approximate \textcolor{black}{right hand side (RHS)} of \eqref{eqn:4c} as

\begin{equation} \label{eqn:8}
\begin{split}
    \frac{ s \ln 2}{b_{r} x_{r}} \leq \ln \left(1+\frac{p_{r}^{i} g_{r}}{b_{r} n_{0}}\right) + \frac{p_{r}^{i} g_{r}}{p_{r} g_{r}+b_{r} n_{0}} \\
    - \frac{(p_{r}^{i} g_{r})^{2}}{p_{r}^{i} g_{r} + b_{r} n_{0}} \frac{1}{p_{r} g_{r}}, \forall{r}.
\end{split}
\end{equation}

So, we solve the following convex problem in iteration $i+1$:
\begin{subequations} 
\begin{align}
\min_{p_{r},f_{r}} \quad & \sum_{k=1}^{K}\left[\max_{r \in \mathcal{R}} \left({\frac{L_{r}C_{r}D_{r}}{f_{r}}}+x_{r}\right.\right. \nonumber\\
& \left.\left.+\sum_{n=1}^{N}\left({\frac{L_{n}C_{n}D_{n}}{f_{n}}}+\frac{(n+1)s}{b_{n}\log_{2}\left(1+\frac{p_{n}g_{n}}{b_{n}n_{0}}\right)}\right)\right)\right] \label{eqn:9a}\\ 
\textrm{s.t.} \quad & \eqref{eqn:4d}, \eqref{eqn:6}, \eqref{eqn:8}. \label{eqn:9b}
\end{align}
\end{subequations}
\textit{For complexity analysis}, this problem consists of $(2R)$ scalar decision variables and $(4R)$ linear or quadratic constraints, which results in the per-iteration computational complexity of $\mathcal{O}\left((2R)^2\sqrt{4R}\right)$ \cite{ben2001lectures}.

\textit{\textbf{For the second block}}, we introduce a new slack variable $y_{m}$ such that:
\begin{equation}
y_{m} \ge \frac{(m+1) s}{b_{m}\log_{2}\left(1+\frac{p_{m}g_{m}}{b_{m}n_{0}}\right)}, \forall{m}. 
\end{equation}
\textbf{Problem in (18)} can be equivalently re-written as

\begin{subequations} 
\begin{align}
\min_{p_{m},f_{m}} \quad & \sum_{k=1}^{K}\left[\max_{r \in \mathcal{R}} \left(\left({\frac{L_{r}C_{r}D_{r}}{f_{r}}}+\frac{s}{b_{r}\log_{2}\left(1+\frac{p_{r}g_{r}}{b_{r}n_{0}}\right)}\right)\right.\right. \nonumber\\
& \left.\left.+\sum_{m=1}^{M}\left({\frac{L_{m}C_{m}D_{m}}{f_{m}}}+y_{m}\right)\right)\right] \label{eqn:11a}\\ 
\textrm{s.t.} \quad & L_{m}\zeta_{m}C_{m}D_{m}{f_{m}^{2}}+y_{m}p_{m} \le E_{m}^{\max}, \forall{m} \label{eqn:11b}\\
& \frac{(m+1)s}{b_{m}y_{m}} \leq \log_{2}\left(1+\frac{p_{m} g_{m}}{b_{m} n_{0}}\right), \forall{m}. \label{eqn:11c}\\
& \eqref{eqn:2c}, \eqref{eqn:2e}. \label{eqn:11d} 
\end{align}
\end{subequations}

\noindent Although the objective function \eqref{eqn:11a} and constraint in \eqref{eqn:11d} are convex, constraints \eqref{eqn:11b} and \eqref{eqn:11c} are still non-convex. To convexify these two constraints, we follow the same strategy for constraints \eqref{eqn:4b} and \eqref{eqn:4c}.

\textit{Constraint \eqref{eqn:11b}}: Similar to constraint \eqref{eqn:4b}, constraint \eqref{eqn:11b} can be convexified as
\begin{equation} \label{eqn:12}
    L_{m}\zeta_{m}C_{m}D_{m}{f_{m}^{2}} +\frac{1}{2} \frac{p_{m}^{i}}{y_{m}^{i}} y_{m}^{2} + \frac{1}{2} \frac{y_{m}^{i}}{p_{m}^{i}} p_{m}^{2} \le E_{m}^{max}, \forall{m}
\end{equation}
\noindent
where $p_{m}^{i}$ and $y_{m}^{i}$ are the feasible point of $p_{m}$ and $y_{m}$ at SCA iteration $i$. 

\textit{Constraint \eqref{eqn:11c}}: Similar to  constraint \eqref{eqn:4c}, we approximate RHS of \eqref{eqn:11c} as
\begin{equation} \label{eqn:13}
\begin{split}
    \frac{(m+1) s \ln{2}}{y_{m}} \le b_{m} \left(1+\frac{p_{m}^{i} g_{m}}{b_{m} n_{0}}\right) + \frac{p_{m}^{i} g_{m}}{p_{m} g_{m}+b_{m} n_{0}} \\
    - \frac{(p_{m}^{i} g_{m})^{2}}{p_{m}^{i} g_{m} + b_{m} n_{0}} \frac{1}{p_{m} g_{m}}, \forall{m}.
\end{split}
\end{equation}
Thus, we solve the following convex problem in iteration $i+1$:
\begin{subequations} 
\begin{align}
\min_{p_{m},f_{m}} \quad & \sum_{k=1}^{K}\left[\max_{r \in \mathcal{R}} \left(\left({\frac{L_{r}C_{r}D_{r}}{f_{r}}}+\frac{s}{b_{r}\log_{2}\left(1+\frac{p_{r}g_{r}}{b_{r}n_{0}}\right)}\right)\right.\right. \nonumber\\
& \left.\left.+\sum_{m=1}^{M}\left({\frac{L_{m}C_{m}D_{m}}{f_{m}}}+y_{m}\right)\right)\right] \label{eqn:14a}\\ 
\textrm{s.t.} \quad & \eqref{eqn:11d}, \eqref{eqn:12}, \eqref{eqn:13}. \label{eqn:14b}
\end{align}
\end{subequations}
\textit{For complexity analysis}, this problem consists of $2M$ scalar decision variables and $4M$ linear or quadratic constraints, which results in the per-iteration computational complexity of $\mathcal{O}\left((2M)^2\sqrt{4M}\right)$ \cite{ben2001lectures}. To summarize, we jointly solve the above two blocks to obtain the solutions for \textbf{problem in (18)}, as illustrated in Algorithm 2.


\begin{algorithm}
\footnotesize
	\caption{Proposed optimization algorithm to minimize FL system latency in load forecasting}
	\begin{algorithmic}[1]
		\label{algo: optimization}
		\STATE \textbf{Input:}  Set the iteration index $i=0$;
		\STATE \textbf{Initialization:} a feasible solution
            $(p_{r}^{0}$, $f_{r}^{0}$, $p_{m}^{0}$, $f_{m}^{0})$ for the problem in (18);
		\STATE \textbf{Repeat}
                \STATE Set $i \gets i+1$
                \STATE Solve problem (25) to update $p_{r}^{i}$, $f_{r}^{i}$;
                \STATE Solve problem (30) to update $p_{m}^{i}$, $f_{m}^{i}$;
            \STATE \textbf{Until} convergence.
        \STATE \textbf{Output:} Optimal $\pmb{p_{r}^{*}}$, $\pmb{f_{r}^{*}}$, $\pmb{p_{m}^{*}}$, $\pmb{f_{m}^{*}}$.
  \end{algorithmic}
\end{algorithm}

\section{Simulations and Performance Evaluation} \label{Sec:SimulationsandPerformanceEvaluation}
\subsection{Environment Settings}
For our load forecasting simulations, we used the \textit{Individual Household Electric Power Consumption} dataset \cite{hebrail2012individual}. It is both multivariate and time-series real-life data focused on physics and chemistry. It describes the electricity consumption for a single household over 47 months with seven features for multivariate analysis, from December 2006 to November 2010. The house is in Sceaux, 7km from Paris, France. It has a total of 2,075,259 instances with nine features. For each feature, there are a total of seven variables and two other non-variables: Date and Time. The variables include (1) global active power (household consumption of total active power), (2) global reactive power (household consumption of total reactive power), (3) voltage (the average voltage (volts) in that household), (4) global intensity (the average intensity (apms) in that household), (5) kitchen's active energy (watt-hours), (6) laundry's active energy (watt-hours), and (7) climate control system's active energy (watt-hours).

\begin{figure}[!t]
  \centering
  \begin{subfigure}{.52\columnwidth}
    \centering
    \includegraphics[width=\linewidth]{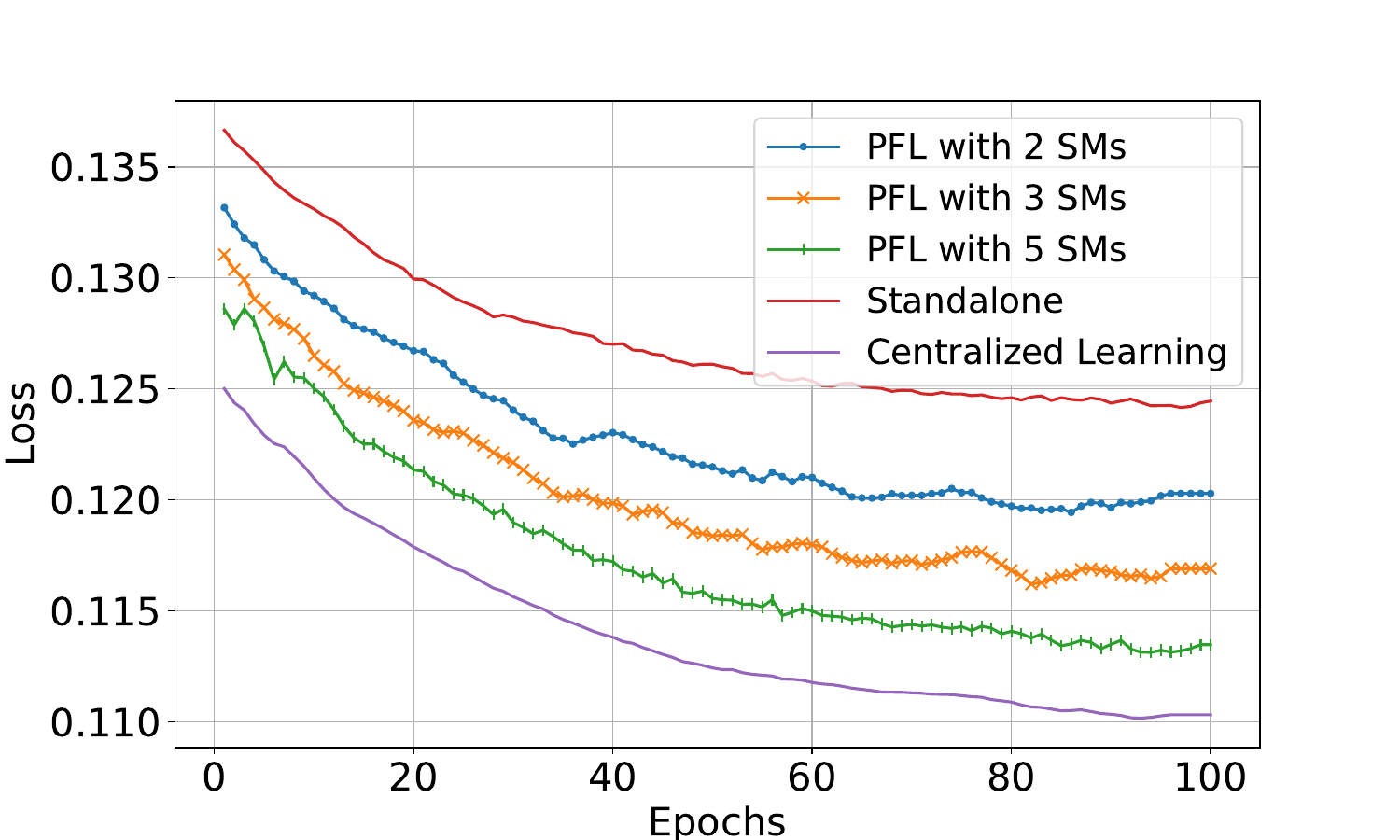}
    \caption{MAE Loss}
    \label{fig:sub1-iid}
  \end{subfigure}%
  \hspace{-0.5cm}
  \begin{subfigure}{.52\columnwidth}
    \centering
    \includegraphics[width=\linewidth]{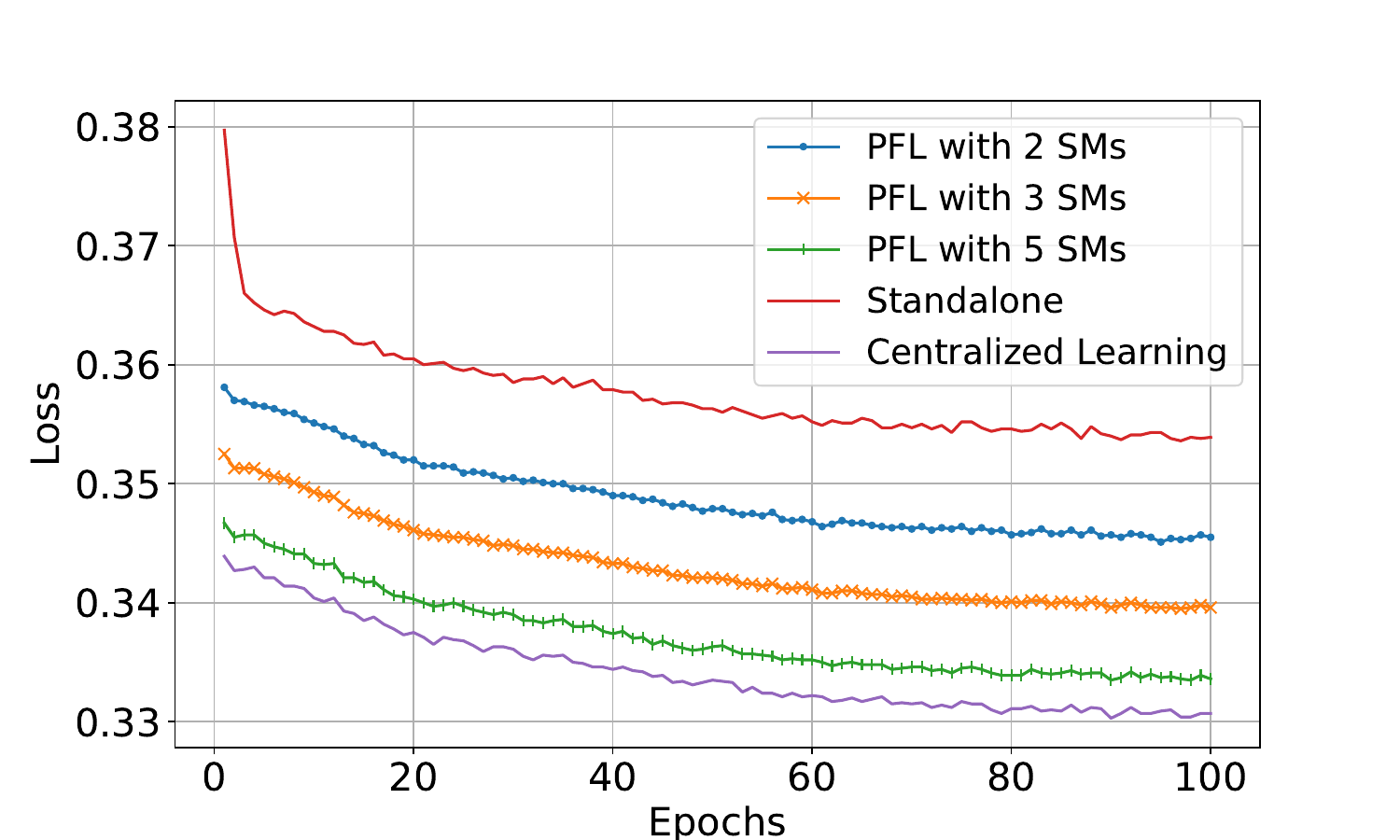}
    \caption{RMSE Loss}
    \label{fig:sub2-iid}
  \end{subfigure}
  \caption{Comparison between different numbers of FL clients (i.e., SMs), standalone, and centralized scheme for IID data.}
  \label{fig: iid}
  \vspace{-2mm}
\end{figure}

\begin{figure}
  \centering
  \begin{subfigure}{.52\columnwidth}
    \centering
    \includegraphics[width=\linewidth]{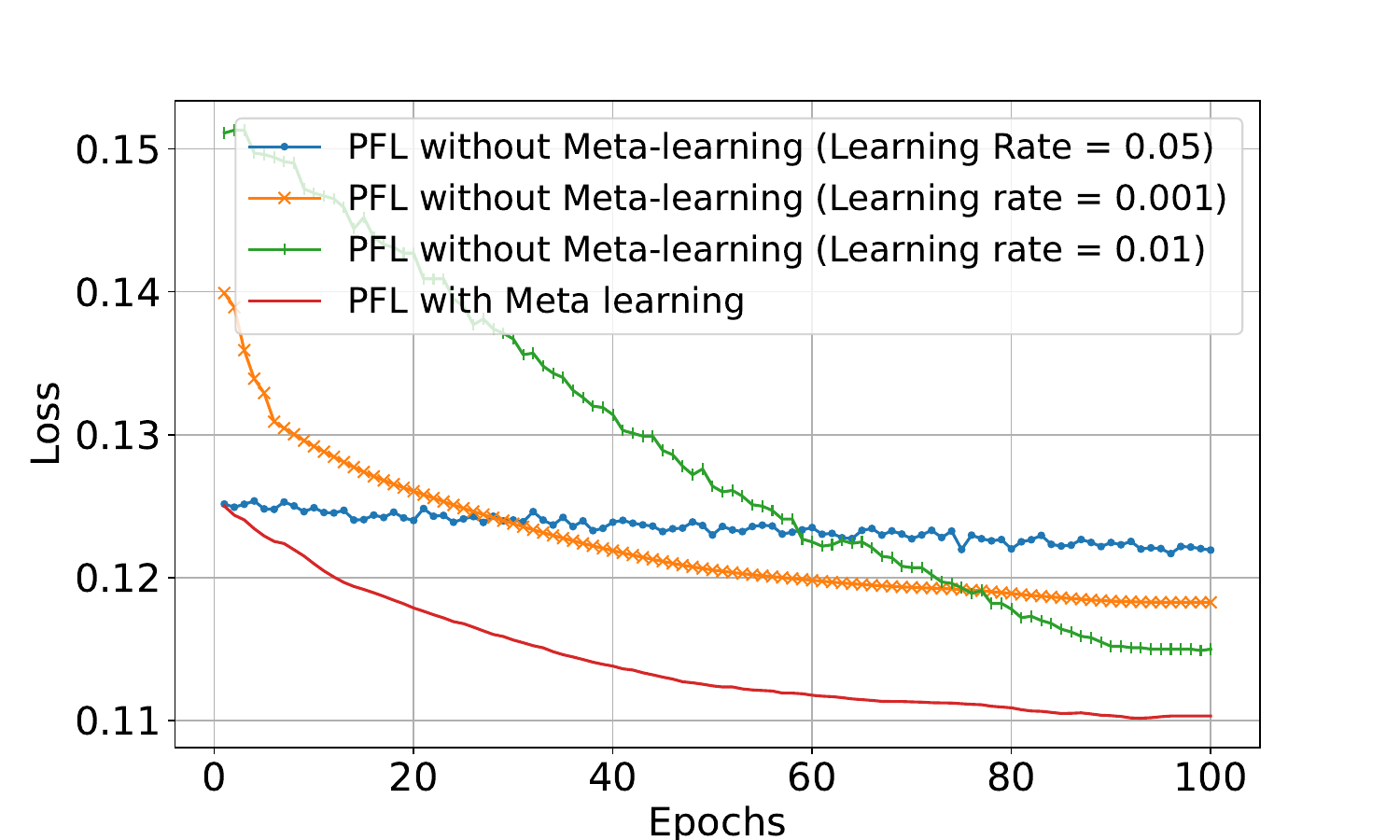}
    \caption{MAE Loss}
    \label{fig:sub1-lr}
  \end{subfigure}%
  \hspace{-0.5cm}
  \begin{subfigure}{.52\columnwidth}
    \centering
    \includegraphics[width=\linewidth]{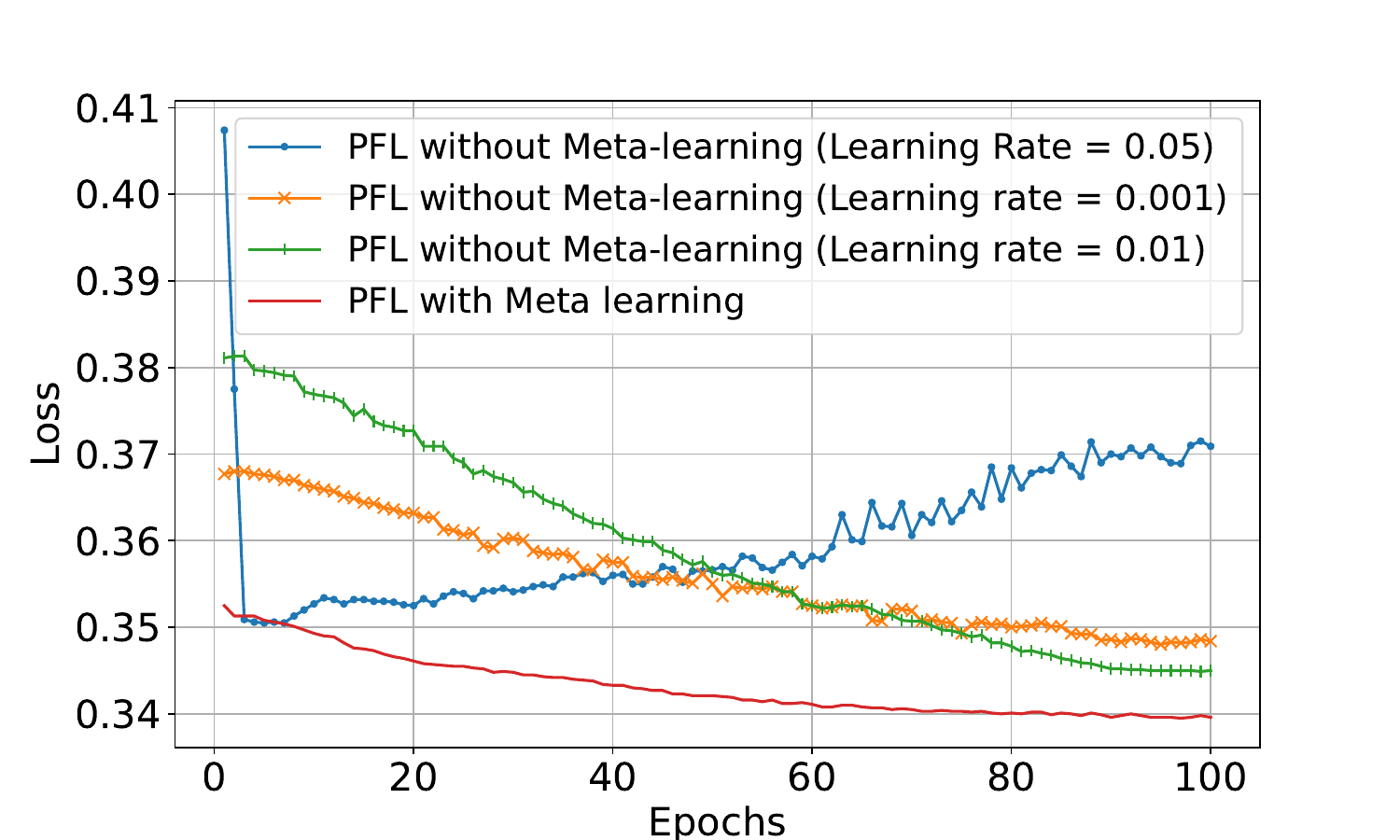}
    \caption{RMSE Loss}
    \label{fig:sub2-lr}
  \end{subfigure}
  \caption{Comparison between different learning rates and meta-learning for 5 SMs non-IID data.}
  \label{fig: lr}
  \vspace{-5mm}
\end{figure}

\begin{figure}
  \centering
  \begin{subfigure}{.52\columnwidth}
    \centering
    \includegraphics[width=\linewidth]{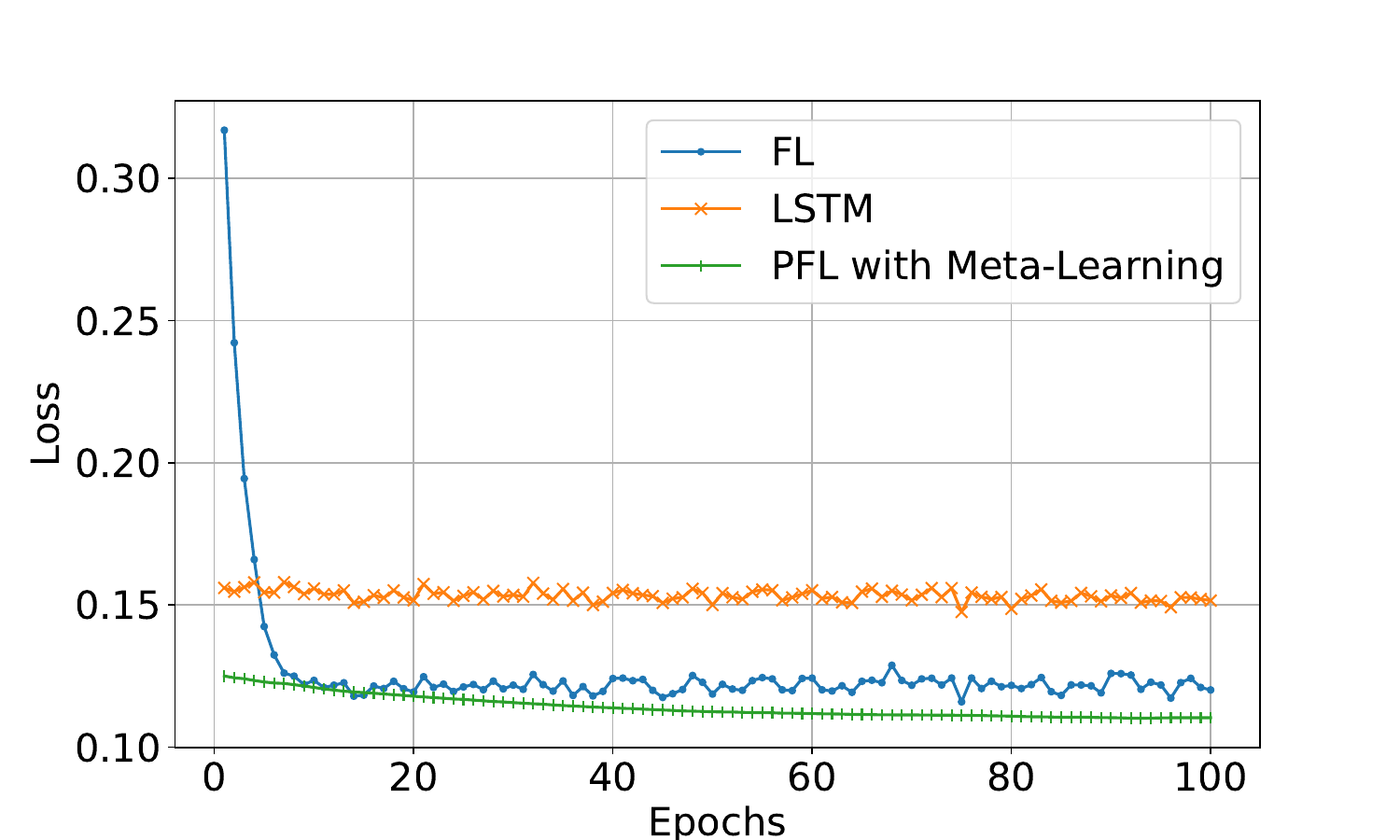}
    \caption{MAE Loss}
    \label{fig:sub1-comp}
  \end{subfigure}%
  \hspace{-0.5cm}
  \begin{subfigure}{.52\columnwidth}
    \centering
    \includegraphics[width=\linewidth]{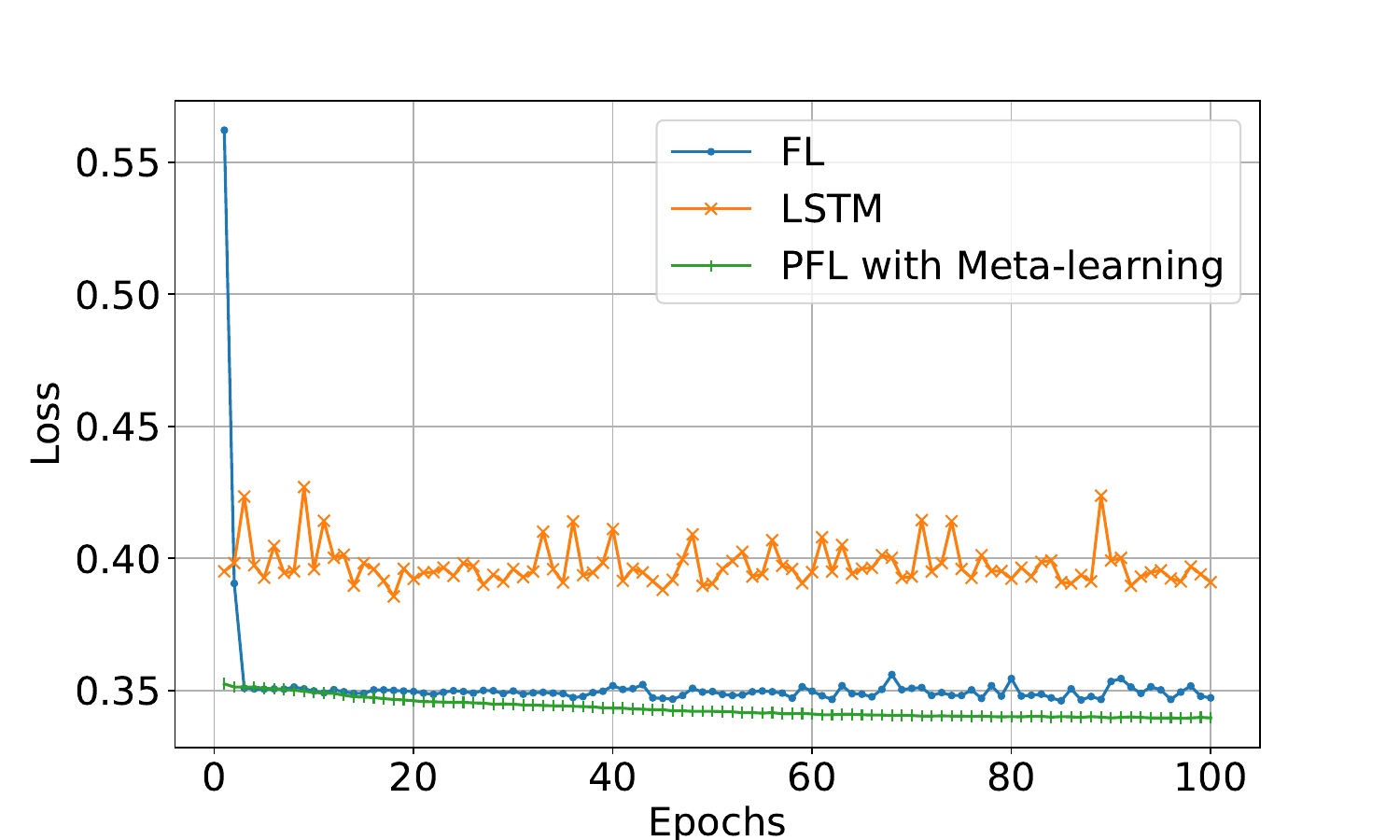}
    \caption{RMSE Loss}
    \label{fig:sub2-comp}
  \end{subfigure}
  \caption{Comparison between state-of-the-art approaches (LSTM and FL) and our approach.}
  \label{fig: comp}
  \vspace{-1mm}
\end{figure}

\textbf{Non-IID SMs distribution.}
\textcolor{black}{For our research, first, we compared our results with 2, 3, and 5 SMs IID data distribution. Then we implemented the comparison using 5 SMs in the non-IID data distribution. Here, the SMs have different \textit{batch sizes} that show different calculation capabilities and \textit{number of data} that demonstrate different data availability. 
For training the model and creating a gradient, we considered three different learning rates: 0.05, 0.001, and 0.0001. A meter could temporarily test its performance for every global epoch based on all three learning rates for 10 local rounds. Then, judging by the performance, we selected the optimal learning rate (the learning rate that produces the lowest loss value) and used that for local training for 10 local rounds. We run our simulation results for 100 global rounds. }

\begin{figure*}
  \centering
  \begin{subfigure}{.66\columnwidth}
    \centering
    \includegraphics[width=\linewidth]{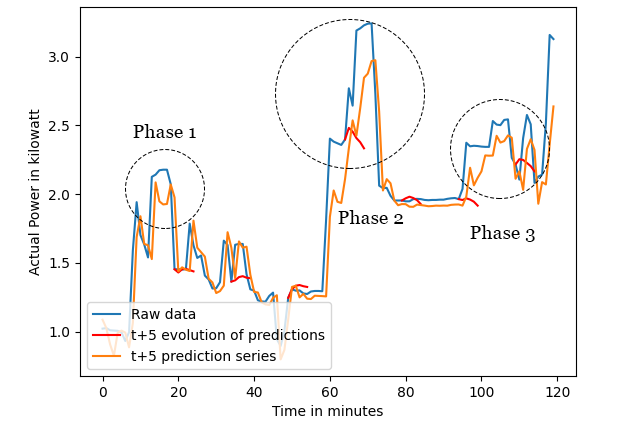}
    \caption{LSTM Result}
    \label{fig:sub1}
  \end{subfigure}%
  \begin{subfigure}{.66\columnwidth}
    \centering
    \includegraphics[width=\linewidth]{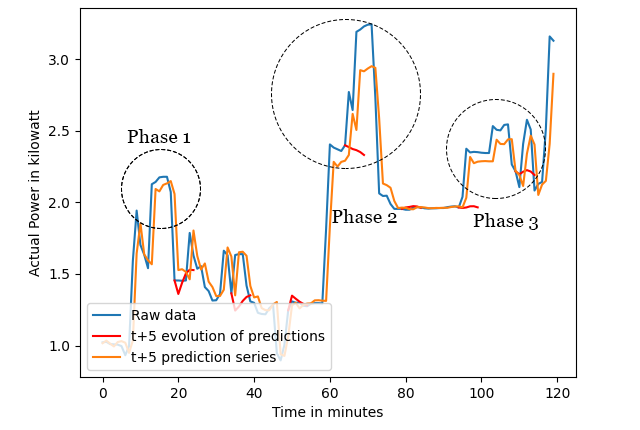}
    \caption{FL Result}
    \label{fig:sub2}
  \end{subfigure}%
  \begin{subfigure}{.66\columnwidth}
    \centering
    \includegraphics[width=\linewidth]{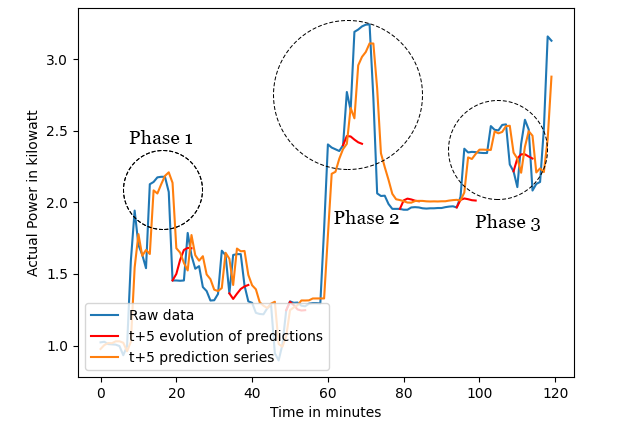}
    \caption{PFL Result}
    \label{fig:sub3}
  \end{subfigure}
  \caption{Simulation result of the original and predicted values for the first 120 minutes in the testing dataset.}
  \label{fig: ovp}
  \vspace{-5mm}
\end{figure*}

\textbf{Local model.} There are multiple layers in the design of the LSTM model for load prediction. Sequences of shape (batch\_size, 24, 1) are first entered into the input layer; each sequence comprises 24 hourly load values. An LSTM layer comprising 50 units comes next, processing the sequential input and capturing temporal dependencies. The next step is to add a dropout layer, whose dropout rate is 0.2. During training, it randomly sets 20\% of the LSTM layer's outputs to zero to prevent over-fitting. If more complicated patterns need to be captured, an additional LSTM layer can be added after this. Lastly, the expected load value for the following hour is generated by adding a completely connected dense layer with a single neuron. This architecture uses the long-term relationships that the LSTM can maintain in the data, making it appropriate for precise load forecasting.

For our load forecasting latency simulations, we considered a multi-hop smart metering network with three routes containing 2, 3, and 4 relay nodes, respectively, and one leaf node per route (unless stated otherwise). Simulation parameters are set to practical values: 20 MHz bandwidth, transmit power $P_r$ and $P_m$ in [5–25] dBm, noise power density $N_0$ = –174 dBm/Hz, and CPU frequencies $F_r = F_m = 2$ GHz \cite{yang2020energy}. Hardware coefficients are set as $\zeta_r = \zeta_m = 10^{-28}$, with local iteration counts $L_r = 5$ for leaf nodes and $L_m = 15$ for relay nodes \cite{yang2020energy}. Simulations were run in MATLAB using YALMIP and MOSEK. We evaluated our joint optimization scheme against two baselines: (i) leaf-only and (ii) relay-only optimization.

\subsection{ Simulation Results for Load Forecasting}
\subsubsection{Training Performance}
\textcolor{black}{At first, Table~\ref{tab:statistical_vs_lstm} compares conventional statistical approaches to the LSTM model in terms of MAE and RMSE loss after 100 training epochs.  As illustrated, statistical approaches like ARIMA \cite{tarmanini2023short}, exponential smoothing \cite{smyl2023drnn}, and regression-based forecasting \cite{madhukumar2022regression} provide larger error values than the LSTM model \cite{bouktif2018optimal}.  This demonstrates the limitations of statistical approaches for capturing complicated temporal relationships and nonlinear patterns in energy consumption data.  By using its recurrent design and memory capacities, the LSTM model achieves improved performance and flexibility, making it more appropriate for modern smart grid forecasting applications.}

\begin{table}
\color{black}
\centering
\footnotesize
\caption{Comparison of Statistical Methods with ML approach (LSTM) for Load Forecasting.}
\label{tab:statistical_vs_lstm}
\begin{tabular}{|l|c|c|}
\hline
\textbf{Method} & \textbf{MAE Loss} & \textbf{RMSE Loss} \\
\hline
ARIMA \cite{tarmanini2023short} & 0.1880 & 0.4125 \\
Exponential Smoothing \cite{smyl2023drnn} & 0.1750 & 0.4001 \\
Regression-based Forecasting \cite{madhukumar2022regression} & 0.1825 & 0.4053 \\
\textbf{LSTM (Centralized)} \cite{bouktif2018optimal} & \textbf{0.1515} & \textbf{0.3910} \\
\hline
\end{tabular}
\end{table}

We implement our proposed personalized Meta-LSTM FL algorithm in IID settings for different SMs. Fig. \ref{fig: iid} describes the result for both MSE loss (a) and RMSE loss (b) for 2,3, and 5 SMs as well as standalone and centralized schemes. 
The figure shows that, as expected, the centralized scheme performs the best while the standalone performs the worst for both loss values. Moreover, the loss value reduces with increased SMs, suggesting improved performance. As a result, we conclude that our approach is adaptable for more SMs and proportionate to the number of SMs. The following simulation results have been obtained using non-IID data from five SMs, as using five SMs has provided the best outcomes. 

\begin{figure}[!ht]
    \centering
    \includegraphics[width=0.45\textwidth]{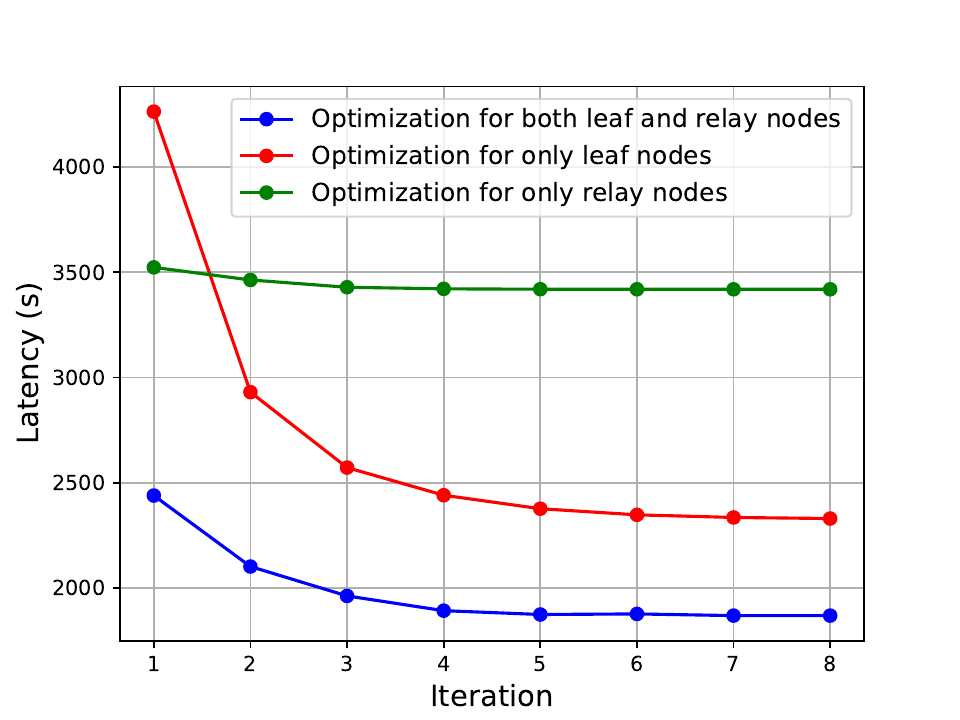}
    \caption{Comparison of training performance. Our joint optimization method can achieve up to 45.33\% lower latency than baselines. }
    \label{fig:7}
\end{figure}
\begin{figure*}[!ht]
    \centering
    \footnotesize
    \begin{subfigure}[t]{0.23\linewidth}
        \centering
        \includegraphics[width=\linewidth]{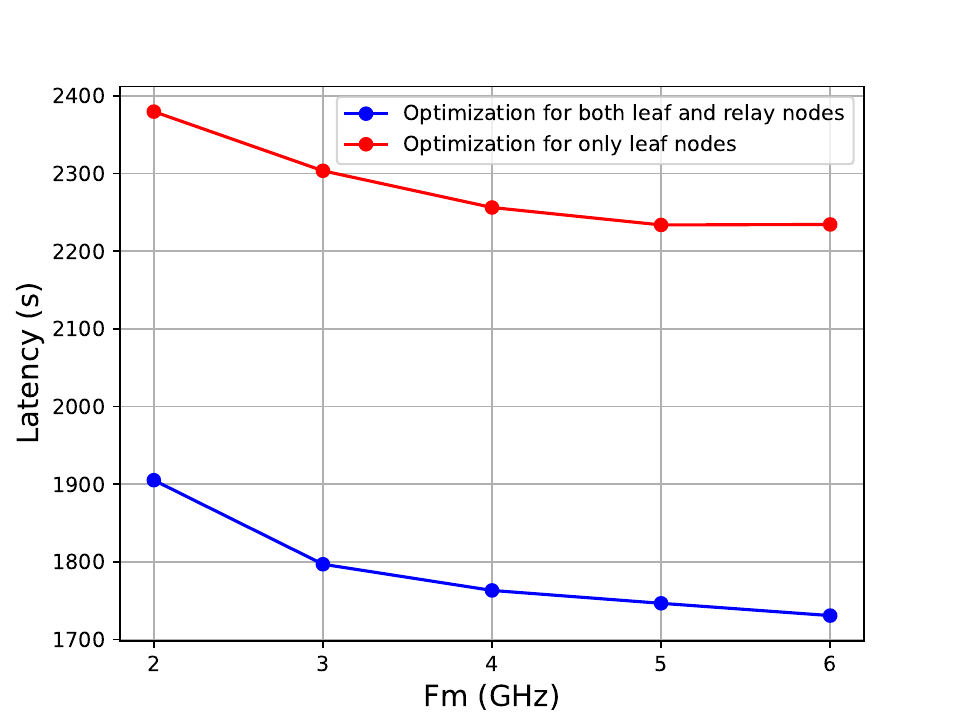}
        \caption{System latency versus maximum frequency of leaf nodes.}
        \label{fig:8a}
    \end{subfigure}
    ~
    \begin{subfigure}[t]{0.23\linewidth}
        \centering
        \includegraphics[width=\linewidth]{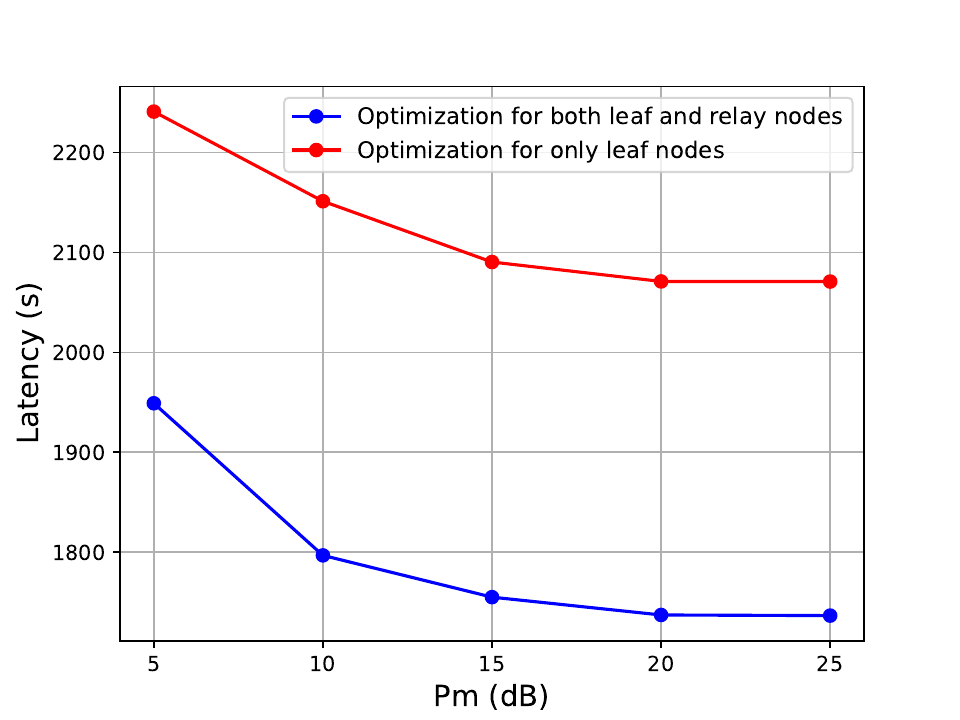}
        \caption{System latency versus maximum transmit power of leaf nodes.}
        \label{fig:8b}
    \end{subfigure}
    \begin{subfigure}[t]{0.23\linewidth}
        \centering
        \includegraphics[width=\linewidth]{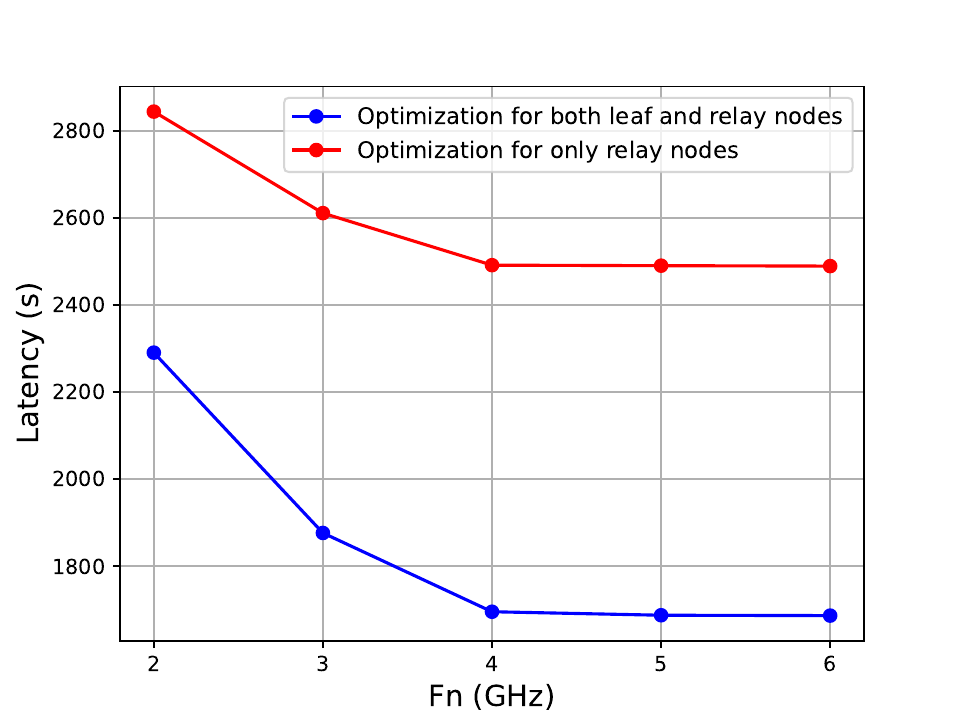}
        \caption{System latency versus maximum frequency of relay nodes.}
        \label{fig:8c}
    \end{subfigure}
    ~
    \begin{subfigure}[t]{0.23\linewidth}
        \centering
        \includegraphics[width=\linewidth]{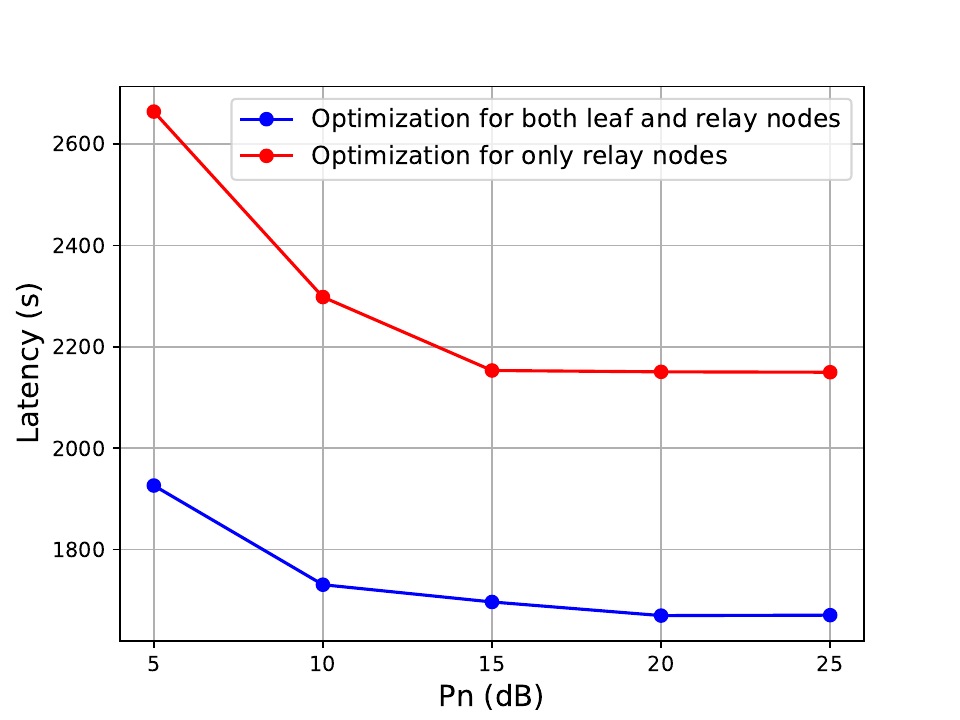}
        \caption{System latency versus maximum transmit power of relay nodes.}
        \label{fig:8d}
    \end{subfigure}
    \caption{Comparison of system latency with different schemes.}
    \vspace{-7mm}
\end{figure*}
Fig. \ref{fig: lr} shows the simulation result for using different learning rates and meta-learning.  The reason why meta-learning works better than other individual learning methods is evident in that figure. While learning rates at 0.05 get off to the best starts, but fail to maintain and overestimate both loss values. Learning rates at 0.001 are too slow to catch up, and communication overheads grow. It is noteworthy to mention that the learning rate of 0.01 did not have a good start to catch up. Nevertheless, meta-learning can extract all the beneficial features from these learning rates through appropriate use, which makes it perfect for more accurate results with less communication overhead. 

\subsubsection{Load Forecasting Performance}
Finally, in Fig. \ref{fig: comp} we have compared our performance with the LSTM algorithm \cite{bouktif2018optimal} and the state-of-the-art FL algorithm \cite{fekri2022distributed} based on global epochs. The figure shows that among the three methods, LSTM performed the worst, as expected. PFL performed somewhat better than FL in this comparison. Furthermore, yields a far more stable result than FL. 

We present the testing findings for the first 120 minutes for three methods in Fig. \ref{fig: ovp}. We concentrate on the three crucial stages, phases 1, 2, and 3. The graphs show the raw data and predictions calculated by different ML models and their evolutions. \textcolor{black}{These phases are chosen because during these phases we can see clear prediction differences between LSTM, FL, and PFL approaches.} In this case, the orange line indicates the expected outcome while the blue line represents the raw data. As a result, the loss value is represented by the space and variations between the orange and blue lines. From Fig \ref{fig: ovp}(a) we can observe that the LSTM technique is not very good at predicting changes, particularly in phases 2 and 3, which have many curves. In phase 1, FL and PFL are far more accurate at predicting. Still, they have not accurately forecasted phase 2's peak. Furthermore, we can observe the variations in phase 2, particularly in phase 3, if we compare the FL and PFL results in Figs. \ref{fig: ovp}(b-c). PFL achieved the highest point in phase 2 and noticeably improved performance in phase 3. 

\textcolor{black}{\subsection{Simulation Results for Load Forecasting Latency }}
\textbf{Trade-off Between Computational Complexity and Latency Reduction.} Fig.~\ref{fig:7} compares our proposed algorithm with scheme 1 and scheme 2, depicting the latency (s) versus the number of iterations. The results demonstrate the superior performance of our proposed algorithm. \textcolor{black}{While the proposed meta-learning-based PFL approach minimizes operational delay by optimizing resource allocation among SMs and relay nodes, it adds significant computational complexity. The computational complexity per iteration is bounded as \(\mathcal{O}((2R)^2 \sqrt{4R}) \) and \(\mathcal{O}((2M)^2 \sqrt{4M}) \), where \(R \) and \(M \) represent the number of leaf and relay nodes, respectively.  As the network grows in size, the per-iteration complexity and overall convergence time increase dramatically.} However, It is evident from the graph that our scheme reaches a stable latency level after the fifth iteration, significantly outperforming the other two schemes in terms of minimizing latency. Specifically, our proposed scheme achieves a 19.79\% reduction in latency compared to scheme 1 and a 45.33\% reduction compared to scheme 2.

We next investigate the latency performance of different schemes. Fig.~\ref{fig:8a} indicates the latency (in seconds) versus the maximum frequency (in GHz) of a leaf node, comparing our proposed algorithm with scheme 1. As higher frequencies generally allow higher data rates, latency decreases with increased maximum frequency of leaf nodes. Both schemes experience reduced latency with higher frequencies, but our proposed algorithm achieves approximately 22.54\% lower latency than scheme 1. This superior performance is due to the algorithm's dynamic adaptation to network conditions, considering both leaf and relay nodes for more efficient resource utilization and minimized latency.

\textcolor{black}{Moreover, Fig.~\ref{fig:8b} illustrates the latency (in seconds) versus the maximum transmit power of a leaf node, comparing our proposed approach with scheme 1.  Our scheme achieves approximately 16.15\% lower latency than scheme 1, despite both schemes benefiting from reduced latency with increased transmit power. Our scheme optimizes resource allocation and communication parameters for leaf and relay nodes, resulting in lower latency and surpassing scheme 1, which does not fully optimize relay node parameters.} \textcolor{black}{As the maximum frequency (GHz) of a relay node increases (as shown in Fig.~\ref{fig:8c}), our algorithm achieves 32.25\% lower latency compared to scheme 2, even though both schemes benefit from the frequency increase. By optimizing resource allocation and communication parameters for both leaf and relay nodes, our scheme ensures more efficient resource use and, consequently, lower latency. This comprehensive approach surpasses scheme 1, which may not fully consider the impact of optimizing relay node parameters on network performance.}

Finally, we evaluate the latency versus the maximum transmit power of a relay node in Fig.~\ref{fig:8d}, highlighting the performance contrast between scheme 2 and our proposed method.
Though latency decreases with increased maximum transmit power of relay nodes in both schemes, our scheme outperforms scheme 2 and achieves 22.29\% lower latency.




\section{Conclusion} \label{Sec:Conslusion}
This paper has proposed a novel PFL approach for load forecasting in smart metering networks.  We have developed a meta-learning algorithm for better load forecasting model training at SMs given non-IID data.  We have also developed an optimization solution to minimize the system latency for the PFL-based load forecasting system by considering the resource allocation of SMs. Simulation results indicate that our proposed PFL method has outperformed existing LSTM and FL approaches with better load forecasting accuracy. Our optimization solution has also achieved significant latency reductions, with up to 45.33\% lower than baseline methods.

In future work, we will consider incentive mechanisms to speed up load forecasting at SMs further. We will also investigate and extend our FL method to other smart grid domains, such as electricity generation prediction.


\section*{Acknowledgment}
This manuscript has been co-authored by UT-Battelle, LLC under Contract No. DE-AC05-00OR22725 with the U.S. Department of Energy. By accepting the article for publication, the publisher acknowledges that the U.S. Government retains a non-exclusive, paid-up, irrevocable, worldwide license to publish or reproduce the published form of the manuscript or allow others to do so, for U.S. Government purposes. The DOE will provide public access to these results under the DOE Public Access Plan (\url{http://energy.gov/downloads/doe-public-access-plan}). This research was partly sponsored by Oak Ridge National Laboratory’s (ORNL’s) Laboratory Directed Research and Development program and by the DOE. The funders had no role in the study design, data collection and analysis, decision to publish, or preparation of this manuscript.

\ifCLASSOPTIONcaptionsoff
  \newpage
\fi
\bibliographystyle{ieeetr}
\bibliography{bibtex/bib/IEEEexample.bib}

\input{Appendix.tex}

\begin{IEEEbiography}[{\includegraphics[width=1in,height=1.25in,clip,keepaspectratio]{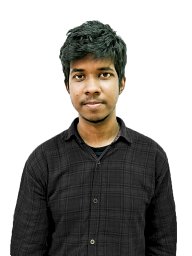}}]{Ratun Rahman}
is a Ph.D. candidate in Electrical and Computer Engineering at The University of Alabama in Huntsville, conducting research in the Networking, Intelligence, and Security Lab. His work focuses on theoretical machine learning, federated learning, and quantum learning. He has published in leading IEEE journals and conferences, aiming to find the optimal machine learning solution to existing problems. Currently, he is working on quantum learning and distributed quantum learning. His research interest focuses on machine learning, federated learning, and quantum learning.
\end{IEEEbiography}

\begin{IEEEbiography}[{\includegraphics[width=1in,height=1.25in,clip,keepaspectratio]{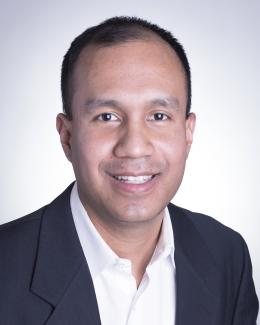}}]{Pablo Moriano}
(Senior Member, IEEE) received B.S. and M.S. degrees in electrical engineering from Pontificia Universidad Javeriana in Colombia and M.S. and Ph.D. degrees in informatics from Indiana University Bloomington, Bloomington, IN, USA. He is a research scientist with the Computer Science and Mathematics Division at Oak Ridge National Laboratory, Oak Ridge, TN, USA. His research lies at the intersection of data science,
network science, and cybersecurity. In particular, he uses data-driven and computational methods to discover, understand, and detect anomalous behavior in large-scale networked systems. Applications of his research range across multiple disciplines, including, the detection of exceptional events in social media, Internet route hijacking, insider threat behavior in version control systems, and anomaly detection in cyber-physical systems. Dr. Moriano is a member of ACM and SIAM.
\end{IEEEbiography}

\begin{IEEEbiography}[{\includegraphics[width=1in,height=1.25in,clip,keepaspectratio]{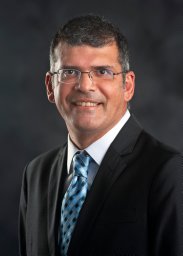}}]{Samee U. Khan}
(Senior Member, IEEE) received a doctorate in 2007 from the University of Texas Arlington. He is professor and head of the Mike Wiegers Department of Electrical and Computer Engineering at Kansas State University. Before joining K-State, he was a faculty member at Mississippi State University (MSU), serving as department head and the James W. Bagley chair from 2020 to 2024. He started his career at North Dakota State University (NDSU) in 2008 and rose through the ranks to become the Walter B. Booth professor. While at NDSU, he was assigned to the National Science Foundation (2016-2020) as cluster lead for computer systems research within the computer and network systems division, where he managed a portfolio of more than: 500 active projects, 700 distinct investigators and \$160 million.
\end{IEEEbiography}

\begin{IEEEbiography}[{\includegraphics[width=1in,height=1.25in,clip,keepaspectratio]{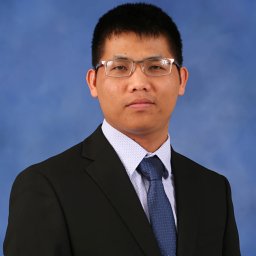}}]{Dinh C. Nguyen}
(Member, IEEE) is an assistant professor at the Department of Electrical and Computer Engineering, The University of Alabama in Huntsville, USA. He worked as a postdoctoral research associate at Purdue University, USA from 2022 to 2023. He obtained the Ph.D. degree in computer science from Deakin University, Australia in 2021. His current research interests include federated machine learning, Internet of Things, wireless networking, and security. He has published over 50 papers (including over 25 first-authored papers) on top-tier IEEE/ACM conferences and journals such as IEEE Journal on Selected Areas in Communications, IEEE Communications Surveys and Tutorials, IEEE Transactions on Mobile Computing, and IEEE Wireless Communications Magazine. He is an Associate Editor of the IEEE Internet of Things Journal, IEEE Open Journal of the Communications Society, and a Lead Guest Editor of IEEE Internet of Things Magazine on the special issue of federated learning for Industrial Internet of Things.  He received the Best Editor Award from IEEE Open Journal of Communications Society in 2023.
\end{IEEEbiography}

\end{document}

%% file: Appendix.tex
\appendix \label{appendix_lemmas}

\section{Proof of Theorem~\eqref{th:main}}\label{app:th}
\label{subsection:Appen}
\subsection{Proof of Lemma 1}
\vspace{-5pt}
\label{Lemma_SGD:upperbound}
From Assumption \eqref{Assump:Variance-gradient}, we have
\begin{equation} 
\footnotesize
\begin{aligned}
& \mathbb{E} ||g_k^j- \bar{g}_k^j||^2 = \mathbb{E} \Big\Vert \frac{1}{N} \sum_{n\in\mathcal{N}} \left(\nabla F_{n}(\boldsymbol{w}_{n,k}^j,\chi_{n,k}^j) - \nabla F_{n}(\boldsymbol{w}_{n,k}^j) \right)\Big\Vert^2 
\\&=\frac{1}{N^2}   \sum_{n\in\mathcal{N}} \mathbb{E} \Big\Vert\left(\nabla F_{n}(\boldsymbol{w}_{n,k}^j,\chi_{n,k}^j) - \nabla F_{n}(\boldsymbol{w}_{n,k}^j) \right)\Big\Vert^2
\leq \frac{\sigma_r^2}{N^2}.
\end{aligned}
\end{equation}
\vspace{-5pt} 
\subsection{Proof of Lemma 2}
\label{Lemma_lemma2:upperbound}
We know that in every global communication round, each UAV performs $J$ rounds of local SGDs where there always exits $j' \leq j$ such that $j-j'\leq J$ and $\boldsymbol{w}_{n,k}^{t'} = \bar{\boldsymbol{w}}_k^{j'}$, $\forall n \in \mathcal{N}$. By using the fact that $\mathbb{E}||X-\mathbb{E}X||^2 = ||X||^2 - ||\mathbb{E}X||^2$ and $\bar{\boldsymbol{w}}_k^j =\mathbb{E}\boldsymbol{w}_{n,k}^j$, we have
\begin{equation} 
\footnotesize
\begin{aligned}
&  \frac{1}{N}\sum_{n\in\mathcal{N}} \mathbb{E} \Big\Vert\bar{\boldsymbol{w}}_k^j-\boldsymbol{w}_{n,k}^j\Big\Vert^2 
= \frac{1}{N}\sum_{n\in\mathcal{N}} \mathbb{E} \Big\Vert\boldsymbol{w}_{n,k}^j - \bar{\boldsymbol{w}}_k^j\Big\Vert^2 
\\&=\frac{1}{N}\sum_{n\in\mathcal{N}}\mathbb{E} \Big\Vert(\boldsymbol{w}_{n,k}^j - \bar{\boldsymbol{w}}_k^{j'}) - (\bar{\boldsymbol{w}}_k^j-\bar{\boldsymbol{w}}_k^{j'}) \Big\Vert^2 
\\&\leq  \frac{1}{N}\sum_{n\in\mathcal{N}} \mathbb{E} \Big\Vert\boldsymbol{w}_{n,k}^j - \bar{\boldsymbol{w}}_k^{j'}\Big\Vert^2  
\leq   \frac{1}{N}\sum_{n\in\mathcal{N}}\mathbb{E}\Big\Vert\left(\sum_{j=j'}^{j-1} (\boldsymbol{w}_{n,k}^j-\bar{\boldsymbol{w}}_k^{j'}) \right) \Big\Vert^2  
\\&=   \frac{1}{N}\sum_{n\in\mathcal{N}}\mathbb{E}\Big\Vert\left(\sum_{j=j'}^{j-1} \eta_k \nabla F_{n}(\boldsymbol{w}_{n,k}^j,\chi_{n,k}^j) \right) \Big\Vert^2 
\\&\leq \frac{1}{N}\sum_{n\in\mathcal{N}}\mathbb{E} \Big\Vert\left(\sum_{j=1}^{j-j'} \eta_k \nabla F_{n}(\boldsymbol{w}_{n,k}^j,\chi_{n,k}^j) \right) \Big\Vert^2  ,
\end{aligned}
\end{equation}
where the last inequality holds since the learning rate $\eta_k$ is decreasing. Using the fact that $||\sum_{j=1}^{U} z^j||^2 \leq U\sum_{j=1}^{U} ||z^j||^2$, $j-j' \leq J$ and assume that $\eta_k^{j'} \leq 2\eta_k$ and $||\nabla F(\boldsymbol{w}_{n,k}^j,\chi_{n,k}^j)||^2 \leq B^2$ for positive constant $B$, we have
\begin{equation} 
\tiny
\begin{aligned}
&  \frac{1}{N}\sum_{n\in\mathcal{N}} \mathbb{E} \Big\Vert\bar{\boldsymbol{w}}_k^j-\boldsymbol{w}_{n,k}^j\Big\Vert^2  
\leq  \frac{1}{N}\sum_{n\in\mathcal{N}}\left( \mathbb{E} \sum_{j=1}^{j-j'} \eta_k^2(j-j')  \Big\Vert \nabla F_{n}(\boldsymbol{w}_{n,k}^j,\chi_{n,k}^j)\Big\Vert^2 \right) 
\\&\leq  \frac{1}{N}\sum_{n\in\mathcal{N}}\left( \mathbb{E} \sum_{j=1}^{j-j'} \eta_k^2J \Big\Vert\nabla F_{n}(\boldsymbol{w}_{n,k}^j,\chi_{n,k}^j)\Big\Vert^2 \right) 
\\&\leq  \frac{1}{N}\sum_{n\in\mathcal{N}}\left({(\eta_k^{j'})}^2J  \sum_{j=1}^{j-j'}  B^2 \right) 
\leq  \frac{1}{N}\sum_{n\in\mathcal{N}} {(\eta_k^{j'})}^2J B^2 \leq  4\eta_kJ B^2.
\end{aligned}
\end{equation}
\vspace{-15pt} 
\subsection{Proof of Lemma 3}
\label{Lemma_lemma3:SGDupdate}
From the SGD update rule $\bar{\boldsymbol{w}}_k^{j+1}  = \bar{\boldsymbol{w}}_k^j - \eta_kg_k^j $ and $||a+b||^2 \leq 2||a||^2 +2||b||^2$ for two real-valued vectors $a$ and $b$, we have 
\begin{equation}\label{eq:main1}
\footnotesize 
\begin{aligned} 
 &||\bar{\boldsymbol{w}}_k^{j+1} - \boldsymbol{w}^*||^2 = ||\bar{\boldsymbol{w}}_k^j - \eta_kg_k^j  - \boldsymbol{w}^*||^2  \leq \\ 
 &\quad \quad \quad \quad \quad \quad \quad \quad \quad \underbrace{||\bar{\boldsymbol{w}}_k^j - \eta_kg_k^j  - \boldsymbol{w}^*||^2}_{(A)}. 
\end{aligned}
\end{equation}
We now focus on the bounding term $(A)$ in~\eqref{eq:main1}. We have
\begin{equation} \label{equa:termA}
\footnotesize
\begin{aligned}
&||\bar{\boldsymbol{w}}_k^j - \eta_kg_k^j  - \boldsymbol{w}^*||^2 = ||\bar{\boldsymbol{w}}_k^j - \eta_kg_k^j - \boldsymbol{w}^* -\eta_k\bar{g}_k^j +\eta_k\bar{g}_k^j||^2 
\\&=||(\bar{\boldsymbol{w}}_k^j - \boldsymbol{w}^* -\eta_k\bar{g}_k^j||^2 + 2\eta_k\langle \bar{\boldsymbol{w}}_k^j - \boldsymbol{w}^* -\eta_k\bar{g}_k^j,\bar{g}_k^j -g_k^j\rangle 
\\& + \eta_k^2||g_k^j - \bar{g}_k^j||^2  
=\underbrace{||(\bar{\boldsymbol{w}}_k^j - \boldsymbol{w}^* -\eta_k\bar{g}_k^j||^2}_{(B)}  + \eta_k^2||g_k^j - \bar{g}_k^j||^2,
\end{aligned}
\end{equation}
where $\langle \bar{\boldsymbol{w}}_k^j - \boldsymbol{w}^* -\eta_k\bar{g}_k^j,\bar{g}_k^j -g_k^j\rangle =0$. We now focus on the bounding term $(B)$. We have \vspace{-5pt}
\begin{equation} \label{equa:boundB1}
\footnotesize 
\begin{aligned} 
&||(\bar{\boldsymbol{w}}_k^j - \boldsymbol{w}^* -\eta_k\bar{g}_k^j||^2 
= ||\bar{\boldsymbol{w}}_k^j - \boldsymbol{w}^*||^2 + \eta_k^2||\bar{g}_k^j||^2 
\\& -2\eta_k \frac{1}{N} \sum_{n\in\mathcal{N}} \langle \bar{\boldsymbol{w}}_k^j - \boldsymbol{w}^*,\nabla F_{n}(\boldsymbol{w}_{n,k}^j)    \rangle
\\& \leq ||\bar{\boldsymbol{w}}_k^j - \boldsymbol{w}^*||^2 + \eta_k^2\frac{1}{N} \sum_{n\in\mathcal{N}}||\nabla F_{n}(\boldsymbol{w}_{n,k}^j)||^2 
\\&-2\eta_k \frac{1}{N} \sum_{n\in\mathcal{N}} \langle \bar{\boldsymbol{w}}_k^j - \boldsymbol{w}_{n,k}^j +\boldsymbol{w}_{n,k}^j - \boldsymbol{w}^*,\nabla F_{n}(\boldsymbol{w}_{n,k}^j)  \rangle
\\& \leq ||\bar{\boldsymbol{w}}_k^j - \boldsymbol{w}^*||^2 + 2\eta_k^2\frac{L}{N} \sum_{n\in\mathcal{N}} (F_{n}(\boldsymbol{w}_{n,k}^j) - F^*) 
\\&-2\eta_k \frac{1}{N} \sum_{n\in\mathcal{N}} \langle \bar{\boldsymbol{w}}_k^j - \boldsymbol{w}_{n,k}^j,\nabla F_{n}(\boldsymbol{w}_{n,k}^j) \rangle
\\&-2\eta_k \frac{1}{N} \sum_{n\in\mathcal{N}} \langle \boldsymbol{w}_{n,k}^j - \boldsymbol{w}^*,\nabla F_{n}(\boldsymbol{w}_{n,k}^j) \rangle,
\end{aligned}
\end{equation}
where we applied $||\sum_{n\in\mathcal{N}} z_n||^2 \leq N\sum_{n\in\mathcal{N}} ||z_n||^2$ in the first inequality, and in the second inequality we applied  L-smoothness  $||\nabla F_{n}(\boldsymbol{w}_{n,k}^j)||^2 \leq 2L (F_{n}(\boldsymbol{w}_{n,k}^j) - F^*)$. For the third term in~\eqref{equa:boundB1}, via Cauchy–Schwarz  inequalities: $2\langle a,b \rangle \leq \frac{1}{\varepsilon}||a||^2 +\varepsilon||b||^2$ for $\varepsilon>0$, we have
\vspace{-5pt}
\begin{equation} 
\footnotesize
\begin{aligned}
& -2 \langle \bar{\boldsymbol{w}}_k^j - \boldsymbol{w}_{n,k}^j,\nabla F_{n}(\boldsymbol{w}_{n,k}^j) \rangle 
= 2 \langle \boldsymbol{w}_{n,k}^j - \bar{\boldsymbol{w}}_k^j,\nabla F_{n}(\boldsymbol{w}_{n,k}^j) \rangle 
\\&\leq \frac{1}{\eta_k} ||\boldsymbol{w}_{n,k}^j - \bar{\boldsymbol{w}}_k^j||^2 + \eta_k||\nabla F_{n}(\boldsymbol{w}_{n,k}^j)||^2
\\& \leq \frac{1}{\eta_k} ||\boldsymbol{w}_{n,k}^j - \bar{\boldsymbol{w}}_k^j||^2 + 2\eta_k L (F_{n}(\boldsymbol{w}_{n,k}^j) - F^*).
\end{aligned}
\end{equation}
 For the last term in~\eqref{equa:boundB1}, by using $\mu$-strong convexity, we have: $\footnotesize \langle \boldsymbol{w}_{n,k}^j - \boldsymbol{w}^*,\nabla F_{n}(\boldsymbol{w}_{n,k}^j) \rangle \geq (F_{n}(\boldsymbol{w}_{n,k}^j) - F^*) +\frac{\mu}{2}||\boldsymbol{w}_{n,k}^j - \boldsymbol{w}^*||^2$. Therefore, \eqref{equa:boundB1} can be rewritten as
\begin{equation} \label{Equa:subB1}
\footnotesize
\begin{aligned}
&||(\bar{\boldsymbol{w}}_k^j - \boldsymbol{w}^* -\eta_k\bar{g}_k^j||^2 
\leq ||\bar{\boldsymbol{w}}_k^j - \boldsymbol{w}^*||^2 + 
 2\eta_k^2\frac{L}{N} \sum_{n\in\mathcal{N}}(F_{n}(\boldsymbol{w}_{n,k}^j) - F^*) 
\\&+ \eta_k\frac{1}{N} \sum_{n\in\mathcal{N}} \left(\frac{1}{\eta_k} ||\boldsymbol{w}_{n,k}^j - \bar{\boldsymbol{w}}_k^j||^2 + 2\eta_k L (F_{n}(\boldsymbol{w}_{n,k}^j) - F^*) \right) 
\\&-2\eta_k \frac{1}{N} \sum_{n\in\mathcal{N}}(F_{n}(\boldsymbol{w}_{n,k}^j) - F^*) -\mu\eta_k \frac{1}{N} \sum_{n\in\mathcal{N}}\frac{\mu}{2}||\boldsymbol{w}_{n,k}^j - \boldsymbol{w}^*||^2
\\&\leq ||\bar{\boldsymbol{w}}_k^j - \boldsymbol{w}^*||^2 + 2\eta_k(2\eta_kL-1) \frac{1}{N} \sum_{n\in\mathcal{N}}(F_{n}(\boldsymbol{w}_{n,k}^j) - F^*) 
\\&+ \frac{1}{N} \sum_{n\in\mathcal{N}}||\bar{\boldsymbol{w}}_k^j-\boldsymbol{w}_{n,k}^j||^2-\mu\eta_k \frac{1}{N} \sum_{n\in\mathcal{N}}||\boldsymbol{w}_{n,k}^j - \boldsymbol{w}^*||^2
\\&= (1-\mu\eta_k)||\bar{\boldsymbol{w}}_k^j - \boldsymbol{w}^*||^2 + 2\eta_k(2\eta_kL-1) \frac{1}{N} \sum_{n\in\mathcal{N}}(F_{n}(\boldsymbol{w}_{n,k}^j) - F^*) 
\\&+ \frac{1}{N}\sum_{n\in\mathcal{N}}||\bar{\boldsymbol{w}}_k^j-\boldsymbol{w}_{n,k}^j||^2,
\end{aligned}
\end{equation}
where we used the fact: $\frac{1}{N} \sum_{n\in\mathcal{N}}||\boldsymbol{w}_{n,k}^j - \boldsymbol{w}^*||^2 = ||\bar{\boldsymbol{w}}_k^j - \boldsymbol{w}^*||^2$. We assume $\eta_k \leq \frac{1}{4L}$, it holds $\eta_kL \leq \frac{1}{4} \Longrightarrow 2\eta_kL -1 \leq -\frac{1}{2}$. Thus we have \vspace{-10pt}
\begin{equation} \label{equa:Cterm}
\footnotesize
\begin{aligned}
&||(\bar{\boldsymbol{w}}_k^j - \boldsymbol{w}^* -\eta_k\bar{g}_k^j||^2 \leq (1-\mu\eta_k)||\bar{\boldsymbol{w}}_k^j - \boldsymbol{w}^*||^2  
\\&+ \frac{1}{N}\sum_{n\in\mathcal{N}}||\bar{\boldsymbol{w}}_k^j-\boldsymbol{w}_{n,k}^j||^2- \underbrace{\frac{1}{2} \frac{1}{N} \sum_{n\in\mathcal{N}}(F_{n}(\boldsymbol{w}_{n,k}^j) - F^*)}_{(C)}.
\end{aligned}
\end{equation}
To bound $(C)$, we have
\vspace{-5pt}
\begin{equation}\scriptsize
\begin{aligned}
&- \frac{1}{N} \sum_{n\in\mathcal{N}}(F_{n}(\boldsymbol{w}_{n,k}^j) - F^*)
= -\frac{1}{N} \sum_{n\in\mathcal{N}}(F_{n}(\boldsymbol{w}_{n,k}^j) - F_n(\bar{\boldsymbol{w}}_k^j)) 
\\&- \frac{1}{N} \sum_{n\in\mathcal{N}}(F_n(\bar{\boldsymbol{w}}_k^j) - F^*)
 \leq - \frac{1}{N} \sum_{n\in\mathcal{N}} \langle \boldsymbol{w}_{n,k}^j - \bar{\boldsymbol{w}}_k^j, \nabla F_n(\bar{\boldsymbol{w}}_k^j)\rangle - (F_{n}(\bar{\boldsymbol{w}}_k^j) - F^*)
\\& \leq  \frac{1}{2}\frac{1}{N} \sum_{n\in\mathcal{N}}\left(\frac{1}{\eta_k}||\boldsymbol{w}_{n,k}^j - \bar{\boldsymbol{w}}_k^j||^2 +\eta_k ||\nabla F_n(\bar{\boldsymbol{w}}_k^j)||^2 \right) - (F_{n}(\bar{\boldsymbol{w}}_k^j) - F^*)
\\& \leq  \frac{1}{2}\frac{1}{N} \sum_{n\in\mathcal{N}}\left(\frac{1}{\eta_k}||\boldsymbol{w}_{n,k}^j - \bar{\boldsymbol{w}}_k^j||^2 +2\eta_kL ||F_n(\bar{\boldsymbol{w}}_k^j) -F^*||^2 \right) - (F_{n}(\bar{\boldsymbol{w}}_k^j) - F^*)
\\& = \frac{1}{2\eta_k}\frac{1}{N} \sum_{n\in\mathcal{N}}||\boldsymbol{w}_{n,k}^j -  \bar{\boldsymbol{w}}_k^j||^2 + \eta_kL\frac{1}{N} \sum_{n\in\mathcal{N}}||F_n(\bar{\boldsymbol{w}}_k^j) -F^*||^2- (F_{n}(\bar{\boldsymbol{w}}_k^j) - F^*)
\\& =\frac{1}{2\eta_k}\frac{1}{N} \sum_{n\in\mathcal{N}}||\boldsymbol{w}_{n,k}^j -  \bar{\boldsymbol{w}}_k^j||^2 + (\eta_kL-1)(F_{n}(\bar{\boldsymbol{w}}_k^j) - F^*).
\end{aligned}
\end{equation}
Therefore,~\eqref{equa:Cterm}
is further expressed as
\begin{equation} \tiny
\begin{aligned}\label{equa:subsubC}
&||(\bar{\boldsymbol{w}}_k^j - \boldsymbol{w}^* -\eta_k\bar{g}_k^j||^2 \leq (1-\mu\eta_k)||\bar{\boldsymbol{w}}_k^j - \boldsymbol{w}^*||^2  + \frac{1}{N}\sum_{n\in\mathcal{N}}||\bar{\boldsymbol{w}}_k^j-\boldsymbol{w}_{n,k}^j||^2 + 
\\&\frac{1}{4\eta_k}\frac{1}{N} \sum_{n\in\mathcal{N}}||\boldsymbol{w}_{n,k}^j -  \bar{\boldsymbol{w}}_k^j||^2 + \frac{1}{2}(\eta_kL-1)(F_{n}(\bar{\boldsymbol{w}}_k^j) - F^*)
\\&\leq (1-\mu\eta_k)||\bar{\boldsymbol{w}}_k^j - \boldsymbol{w}^*||^2  + \frac{1}{N}\sum_{n\in\mathcal{N}}||\bar{\boldsymbol{w}}_k^j-\boldsymbol{w}_{n,k}^j||^2 
\\& + 
\frac{1}{4\eta_k}\frac{1}{N} \sum_{n\in\mathcal{N}}||\boldsymbol{w}_{n,k}^j -  \bar{\boldsymbol{w}}_k^j||^2
\\& = (1-\mu\eta_k)||\bar{\boldsymbol{w}}_k^j - \boldsymbol{w}^*||^2  + \left(1+\frac{1}{4\eta_k}\right) \frac{1}{N}\sum_{n\in\mathcal{N}}||\bar{\boldsymbol{w}}_k^j-\boldsymbol{w}_{n,k}^j||^2,
\end{aligned}
\end{equation}
where we used the fact that $\eta_kL-1 \leq 0$ and $F_{n}(\bar{\boldsymbol{w}}_k^j) - F^* \geq 0$ and thus $\frac{1}{2}(\eta_kL-1)(F_{n}(\bar{\boldsymbol{w}}_k^j) - F^*) \leq 0$. We now plug~\eqref{equa:subsubC} back into~\eqref{equa:termA}, we have
\begin{equation}\footnotesize  \label{equa:subsubC-final}
\begin{aligned}
&||\bar{\boldsymbol{w}}_k^j - \eta_kg_k^j  - \boldsymbol{w}^*||^2 
= (1-\mu\eta_k)||\bar{\boldsymbol{w}}_k^j - \boldsymbol{w}^*||^2  
\\&+ \left(1+\frac{1}{4\eta_k}\right) \frac{1}{N}\sum_{n\in\mathcal{N}}||\bar{\boldsymbol{w}}_k^j-\boldsymbol{w}_{n,k}^j||^2 +\eta_k^2||g_k^j - \bar{g}_k^j||^2.
\end{aligned}
\end{equation}
By plugging \eqref{equa:subsubC-final} into \eqref{eq:main1} and taking expectation we obtain
\begin{equation} \tiny
\begin{aligned}
&\mathbb{E}||\bar{\boldsymbol{w}}_k^{j+1} - \boldsymbol{w}^*||^2  \leq (1-\mu\eta_k)\mathbb{E}||\bar{\boldsymbol{w}}_k^j - \boldsymbol{w}^*||^2  
\\&+ \left(1+\frac{1}{\eta_k}\right) \mathbb{E} \left[ \frac{1}{N}\sum_{n\in\mathcal{N}}\Big\Vert\bar{\boldsymbol{w}}_k^j-\boldsymbol{w}_{n,k}^j\Big\Vert^2  \right]  +\eta_k^2\mathbb{E}||g_k^j - \bar{g}_k^j||^2.  
\end{aligned}
\end{equation}

\subsection{Proof of Proof 1}
\label{Proof_globalbound}
Based on Lemmas 1,2,3,  we have
\begin{equation} \footnotesize \label{equa:updaterule_final}
\begin{aligned}
&\mathbb{E}||\bar{\boldsymbol{w}}_k^{j+1} - \boldsymbol{w}^*||^2  
\leq (1-\mu\eta_k)\mathbb{E}||\bar{\boldsymbol{w}}_k^j - \boldsymbol{w}^*||^2  
\\&+ 4 \left(1+\frac{1}{\eta_k}\right) \eta_kJ B^2 + \frac{\eta_k^2\sigma_r^2}{N^2}
+ \frac{1}{N}\sum_{n\in\mathcal{N}}\frac{H^2}{\epsilon_{n,k}-z}.
\end{aligned}
\end{equation}
Let us define $ Y_k^j = \mathbb{E}||\bar{\boldsymbol{w}}_k^j - \boldsymbol{w}^*||^2$ and $\Phi_k = 4 \left(\frac{\eta_k+1}{\eta_k^2} \right)J B^2 + \frac{\sigma_r^2}{N^2}
+ \frac{1}{N}\sum_{n\in\mathcal{N}}\frac{H^2}{\epsilon_{n,k}-z}$, from \eqref{equa:updaterule_final} we have
\begin{equation}\label{equa:ytransform}
\sum_{t=1}^{T}Y_k^{j+1} \leq \sum_{j=0}^{J-1} (1-\mu\eta_k)Y_k^j + \eta_k^2\Phi_k,
\end{equation}
By $Y_k = \sum_{j=0}^{J-1}Y_k^j $, \eqref{equa:ytransform} is rewritten as
\begin{equation}\label{equa:ykform}
Y_k^{j+1} \leq (1-\mu\eta_k)Y_k^j  + \eta_k^2\Phi_k,   
\end{equation}
We define a diminishing stepsize $\eta_k = \frac{4\theta}{k+\omega}$ for some $\theta >\frac{1}{4\mu}$ and $\omega >0$. By defining $m_k =\max \{\frac{\theta^2\Phi_k}{4\theta\mu-1}, (\omega+1)Y_{k-1}\}$, we prove that $Y_k \leq \frac{m_k}{k+\omega}$ by induction. Due to $4\theta\mu >1$, from \eqref{equa:ykform} we have
\begin{equation} \footnotesize  \label{equa:_final}
\begin{aligned}
&Y_{k+1} = \left(1-\frac{4\theta\mu}{k+\omega} \right) \frac{m_k}{k+\omega} +  \frac{16\theta^2}{(k+\omega)^2}\Phi_k
 \leq \frac{k+\omega-1}{(k+\omega)^2}m_k 
 \\&+ \frac{16\theta^2}{(k+\omega)^2}\Phi_k
- \frac{4\theta\mu-1}{(k+\omega)^2} 
 \leq \frac{k+\omega-1}{(k+\omega)^2}m_k - \frac{4\theta\mu-1}{(k+\omega)^2}
  \\&\leq \frac{k+\omega-4\theta\mu}{(k+\omega)^2 - (4\theta\mu)^2}m_k 
 = \frac{1}{k+\omega+4\theta\mu}m_k \leq \frac{1}{k+\omega+1}m_k.
\end{aligned}
\end{equation}
We choose $\theta = \frac{4}{\mu}$ and $\omega = \frac{L}{\mu}$ , it follows that
\begin{equation}  \footnotesize
\begin{aligned}
&m_k =\max \{\frac{\theta^2\Phi_k}{4\theta\mu-1}, (\omega+1)Y_{k-1}\} \leq \frac{\theta^2\Phi_k}{4\theta\mu-1} + (\omega+1)Y_{k-1} 
\\&= \frac{16\Phi_k}{15\mu^2} + \left( \frac{L}{\mu}+1\right)Y_{k-1}.
\end{aligned}
\end{equation}
By using the $L$-smoothness of $F(.)$ and $\mu$-strong convexity of $F_{n}(\boldsymbol{w}_{k-1})$: $\mathbb{E}||\boldsymbol{w}_{k-1} - \boldsymbol{w}^*||^2 \leq \frac{2}{\mu}(F_{n}(\bar{\boldsymbol{w}}_{k-1}) -F^*)$, we have
\begin{equation} \footnotesize \label{equa:final_convergenceIID}
\begin{aligned}
&\mathbb{E}\left[F_{n}(\bar{\boldsymbol{w}}_k)\right] -F^* \leq \frac{L}{2}Y_k\leq \frac{L}{2}\frac{m_k}{(k+\omega)} 
\\&\leq \frac{L}{2(k+L/\mu)}\left[\frac{16\Phi_k}{15\mu^2} + \left( \frac{L}{\mu}+1\right) \mathbb{E}||\boldsymbol{w}_{k-1} - \boldsymbol{w}^*||^2 \right]
 \\& \leq \frac{L}{2(k+L/\mu)}\left[\frac{16\Phi_k}{15\mu^2} + \left( \frac{2L}{\mu^2}+\frac{2}{\mu}\right) (F_{n}(\bar{\boldsymbol{w}}_{k-1}) -F^*) \right]
\\& = \frac{L(1+L/\mu)}{\mu}\frac{1}{k+L/\mu}(F_{n}(\bar{\boldsymbol{w}}_{k-1}) -F^*) 
\\&+ \frac{16L}{30\mu^2(k+L/\mu)} \left[ 4JB^2 \left(\frac{\eta_k+1}{\eta_k^2} \right)+ \frac{\sigma_r^2}{N^2}  \right].
\end{aligned} 
\end{equation}
Finally, by applying \eqref{equa:final_convergenceIID} recursively over $K$ global rounds, the convergence bound of the federated model training at each modality cluster of $N$ UAVs after $K$ global communication rounds can be given as 
\begin{equation} \footnotesize \label{equa:final_convergenceFunction}
\begin{aligned}
&\mathbb{E}\left[F_n(\boldsymbol{w}_K)\right] -F^* \leq \frac{L(1+L/\mu)}{\mu}\frac{1}{(K+L/\mu)} (F_n(\boldsymbol{w}_1) -F^*)
 \\&+\frac{16L}{30\mu^2(K+L/\mu)}\sum_{k=1}^{K} \left[4JB^2 \left(\frac{\eta_k+1}{\eta_k^2} \right) +\frac{\sigma_r^2}{N^2} \right],
\end{aligned}
\end{equation}
which completes the proof.
